\documentclass{article}

\PassOptionsToPackage{numbers,compress}{natbib}
\usepackage[main,final]{neurips_2026}

\usepackage[utf8]{inputenc} 
\usepackage[T1]{fontenc}    
\usepackage{hyperref}       
\hypersetup{hidelinks}
\usepackage{url}            
\usepackage{booktabs}       
\usepackage{amsfonts}       
\usepackage{nicefrac}       
\usepackage{microtype}      
\usepackage[table]{xcolor}  
\usepackage{multirow} 
\usepackage{graphicx} 
\usepackage{wrapfig}
\usepackage{amsmath,amssymb,amsthm}
\usepackage[table]{xcolor}
\usepackage{subcaption}

\definecolor{tablebg}{RGB}{248,249,250}
\definecolor{groupbg}{RGB}{242,243,245}
\definecolor{ourbg}{RGB}{242,247,255}

\usepackage{booktabs}
\usepackage{multirow}
\usepackage[table]{xcolor}
\usepackage{makecell}

\newtheorem{proposition}{Proposition}
\newtheorem{corollary}{Corollary}

\begin{document}

\title{UniPro: Unified Multi-Mode Medical Image Segmentation from 2D Images to 3D Volumes via Propagation}

\author{%
  Bangwei Guo\textsuperscript{1}\qquad Yunhe Gao\textsuperscript{2}\qquad Meng Ye\textsuperscript{3}\qquad Yang Zhou\textsuperscript{1} \\
  \bfseries Difei Gu\textsuperscript{1}\qquad Guoning Zhang\textsuperscript{1}\qquad Leon Axel\textsuperscript{4}\qquad Dimitris N.~Metaxas\textsuperscript{1} \\[0.5em]
  \normalfont\textsuperscript{1}Rutgers University\qquad\textsuperscript{2}Stanford University \\
  \normalfont\textsuperscript{3}The University of Texas at Arlington\qquad\textsuperscript{4}New York University
}
\maketitle

\begin{abstract}
Medical image segmentation remains fragmented along two axes: segmentation paradigms and data dimensionality. Existing methods are typically developed separately for semantic, in-context, and interactive segmentation, and are further specialized to either native 2D images or 3D volumetric data. In clinical practice, however, segmentation workflows take many forms: a case may be initialized by semantic prediction, reference-guided segmentation, or user interaction. Regardless of how it begins, fine-grained refinement is naturally performed on 2D views; for volumetric scans, such 2D edits must propagate coherently to the rest of the volume. We present \textbf{UniPro}, a \textbf{uni}fied model that bridges segmentation paradigms and data dimensionality, using \textbf{pro}pagation to extend 2D segmentation to 3D volumes. Our key insight is that volumetric propagation and in-context segmentation share the same reference-conditioned prediction mechanism, differing only in whether the reference image--mask pairs come from other cases or from previously segmented neighboring slices. Building on this view, UniPro supports semantic, in-context, interactive, and propagation-based segmentation within a single slice-based framework, using class priors, reference exemplars, user clicks, and neighboring-slice predictions as mode-specific conditioning inputs. To improve propagation reliability, UniPro further incorporates bidirectional and 3D supervision to regularize slice-wise propagation beyond per-slice losses. Extensive experiments across diverse modalities and anatomies show that UniPro achieves strong performance across all segmentation settings, enabling annotation-efficient 3D segmentation from sparse 2D initialization and reducing slice-by-slice correction effort.
\end{abstract}

\section{Introduction}

Medical image segmentation is a cornerstone of quantitative analysis and clinical decision-making. By delineating anatomical structures and pathological regions, it supports a wide range of downstream applications, including structural assessment~\citep{bilic2023liver_lits, wasserthal2023totalsegmentator}, disease diagnosis~\citep{shen2017deep, guo2026learning}, and agent-driven clinical analysis workflows~\citep{li2025co,liu2026medsam}. Recent advances in deep learning have substantially improved segmentation performance across diverse organs, imaging modalities, and datasets~\citep{hatamizadeh2022unetr, ma2024segment}. Despite this progress, medical image segmentation remains fragmented along two axes: segmentation paradigms and data dimensionality. Existing methods are typically developed separately for semantic, in-context, and interactive segmentation~\citep{butoi2023universeg, isensee2021nnu, isensee2025nninteractive, ma2024segment}, and are further specialized to either native 2D images or 3D volumetric data~\citep{du2023segvol, gao2024training, ye2023uniseg, ma2025medsam2, zhang2022bridging}. As a result, advances in individual settings have not yet led to a unified framework that naturally supports the full spectrum of medical segmentation workflows.

This fragmentation is also misaligned with real-world clinical practice. Clinicians work with both native 2D images and volumetric scans, for example reading chest X-rays and dermatology photographs in their original 2D form while reviewing CT and MRI studies slice-by-slice in 3D volumes~\citep{shen2017deep}. A case may be initialized by semantic prediction, reference-guided segmentation, or user interaction~\citep{guo2026kprism}. Regardless of the initialization source, fine-grained refinement is naturally performed on 2D views, where local boundaries and subtle anatomical details are easier to inspect and correct~\citep{ma2025medsam2}. For volumetric data, such 2D edits can be extended to the remaining slices through propagation, which
sequentially predicts segmentation masks using
initialized or previously segmented slices as
references while preserving local editability. Together, these observations suggest that medical segmentation is inherently multi-mode and cross-dimensional, and that flexible initialization combined with editable volumetric propagation can naturally bridge these dimensions.

\begin{figure*}[t]
    \centering
    \includegraphics[width=0.95\textwidth]{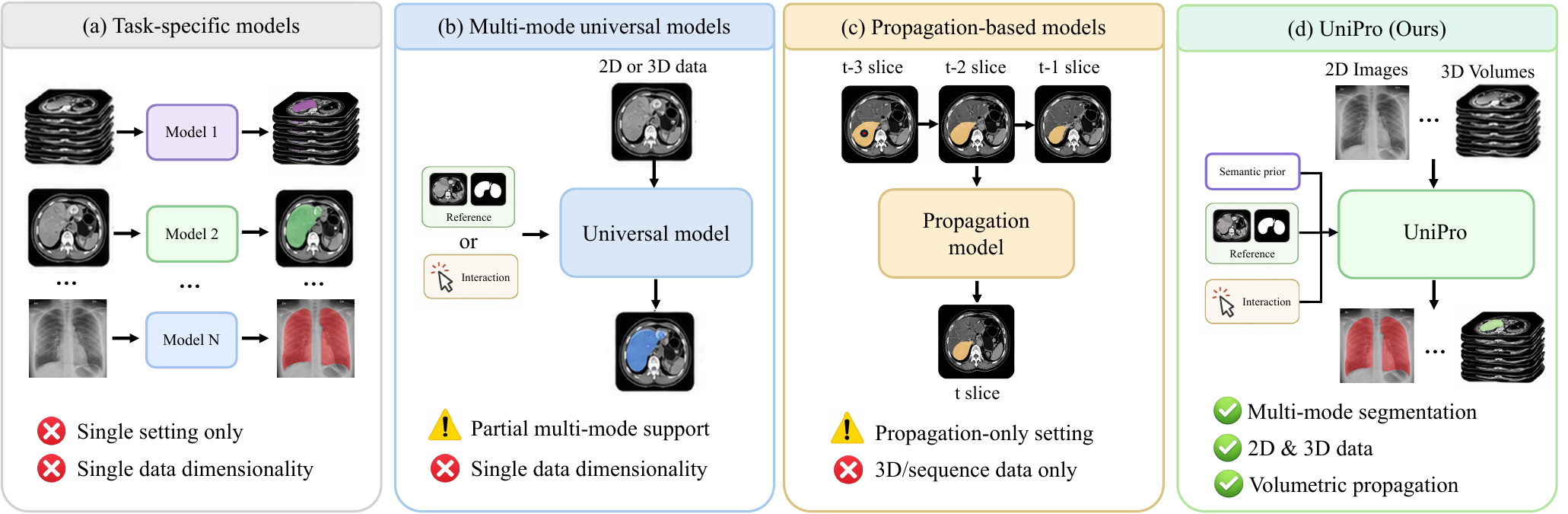}
    \caption{
    Comparison of medical image segmentation paradigms.
    Existing task-specific, universal, and propagation-based models cover only part of the design space, with limitations in segmentation modes, data dimensionality, or propagation capability.
    UniPro unifies multi-mode segmentation, native 2D and 3D data, and volumetric propagation in a single slice-based framework.
    }
    \label{fig1}
    \vspace{-1em}
\end{figure*}

To address this challenge, we propose \textbf{UniPro}\footnote{Project repository: {\urlstyle{same}\url{https://github.com/bangwayne/UniPro}}.}, a \textbf{uni}fied model for multi-mode medical image segmentation, with \textbf{pro}pagation as the bridge between 2D and 3D (Figure~\ref{fig1}). Our key insight is that volumetric propagation and in-context segmentation share the same reference-conditioned prediction mechanism: both predict a target mask conditioned on image--mask references, differing only in whether the references come from other cases or from previously segmented neighboring slices within the same volume. Within a single slice-based framework, UniPro additionally supports semantic and interactive segmentation through class priors and user clicks, thereby unifying class-level, reference-based, click-based, and neighboring-slice conditioning without separate task-specific architectures. Reliable propagation remains challenging because sequential predictions can accumulate drift along the propagation chain, and slice-wise optimization alone does not sufficiently constrain the volumetric consistency. To improve propagation reliability, UniPro incorporates bidirectional and 3D supervision to regularize slice-wise propagation beyond per-slice losses. Together, these components enable UniPro to deliver coherent 3D segmentation from sparse 2D initialization while remaining locally editable—matching the multi-mode, cross-dimensional nature of clinical workflows.

In summary, our contributions are: 
\textbf{(1)} To our knowledge, \textbf{UniPro} is the first unified slice-based framework that jointly supports semantic, in-context, interactive, and propagation-based segmentation across native 2D images and 3D volumes. 
\textbf{(2)} We formulate volumetric propagation and in-context segmentation under a shared reference-conditioned prediction mechanism, allowing propagation to reuse the same conditioning pathway over neighboring slices. 
\textbf{(3)} We incorporate bidirectional and 3D supervision to regularize slice-wise propagation beyond per-slice losses, improving reliability for volumetric medical segmentation. 
\textbf{(4)} We evaluate UniPro on 14 medical imaging datasets across diverse modalities, anatomies, and segmentation settings, demonstrating strong performance and generalization, together with annotation-efficient 3D segmentation from sparse 2D initialization.

\section{Related Work}

\subsection{Universal Medical Image Segmentation}

Recent studies have explored unified medical segmentation frameworks that support multiple segmentation paradigms within a single model. Existing approaches typically achieve partial unification, such as combining semantic and in-context segmentation~\citep{gao2025show}, semantic and interactive segmentation~\citep{guo2025verse, shen2025protosam, ding2024clustering}, or in-context and interactive segmentation~\citep{wong2024multiverseg, zhang2025atlas}. However, these frameworks are usually designed for a subset of segmentation modes and are often confined to a single data dimensionality. In contrast, UniPro aims to unify semantic, in-context, interactive, and propagation-based segmentation while supporting both native 2D images and 3D volumes through a slice-based framework.

\subsection{2D and 3D Medical Image Segmentation}

Medical image segmentation methods are commonly developed for either native 2D images or 3D volumetric scans. 2D models, such as nnU-Net~\citep{isensee2021nnu}, MedSAM~\citep{ma2024segment}, UniverSeg~\citep{butoi2023universeg}, and SAM-Med2D~\citep{cheng2023sam}, are naturally suited for native 2D modalities such as X-ray, dermoscopy, ultrasound, and endoscopy, and can also be applied to 3D volumes by processing them slice by slice. However, when used for volumetric data, independent 2D prediction does not explicitly model cross-slice continuity and often requires dense per-slice prediction or correction to obtain coherent 3D masks. In contrast, 3D models, such as UNETR~\citep{hatamizadeh2022unetr}, nnInteractive~\citep{isensee2025nninteractive}, and SAM-MED3D~\citep{wang2025sam}, capture volumetric context and inter-slice dependencies, but are typically tied to volumetric inputs, often incur higher computational and memory costs than conventional 2D slice-based models, and provide limited flexibility for localized 2D slice-level editing. UniPro bridges this gap by retaining compatibility with native 2D images and slice-level editability, while using propagation to extend sparse 2D initialization and refinements into coherent 3D predictions.

\subsection{Propagation-Based Models}

Propagation-based models segment volumes sequentially by propagating mask information from annotated or previously segmented slices to adjacent slices. Compared with full-volume 3D processing, propagation-based models retain slice-level modeling and local editability through sequential slice-wise prediction~\citep{chen2025pam}.  Correspondence- and registration-based approaches~\citep{yeung2021sli2vol,an2025sli2volplus,bitarafan2022vol2flow} extend a single annotated slice to a full volume by learning
inter-slice correspondences or displacement fields. Recent medical propagation models~\citep{ma2025medsam2, chen2025pam, bai2024revsam2, chen2025accelerating} are often inspired by video segmentation models such as SAM2~\citep{ravi2024sam}. However, they mainly operate on 3D volumes or video-like sequences and are usually designed for propagation-centric, interaction-based segmentation, rather than jointly supporting native 2D images and multiple segmentation modes. Moreover, most propagation methods lack explicit regularization for anatomical continuity in medical volumes, making sequential predictions prone to error accumulation and cross-slice drift~\citep{bai2024revsam2}. In contrast, UniPro integrates propagation into a unified multi-mode framework by treating neighboring-slice propagation as reference-conditioned prediction, and further regularizes slice-wise propagation with bidirectional and 3D supervision.

\section{Method}

\subsection{Overview}
\label{sec:overview}

Figure~\ref{fig2} illustrates the UniPro framework. UniPro is built on a shared 2D encoder--decoder backbone that supports four segmentation settings: semantic, in-context, interactive, and propagation-based segmentation through different forms of conditioning applied to the same backbone. Specifically, UniPro supports the following settings:

\begin{itemize}
    \item \textbf{Semantic segmentation}: the model predicts the mask from the input image conditioned on class-level prior queries.
    
    \item \textbf{In-context segmentation}: the model conditions on reference image--mask pairs for reference-guided prediction.
    
    \item \textbf{Interactive segmentation}: the model takes user clicks for mask prediction and iterative refinement.
    
    \item \textbf{Propagation for 3D volumes}: the model performs slice-wise propagation from a seed slice (initialized by any of the above modes), treating each query slice as in-context segmentation conditioned on previously segmented neighboring slices.

\end{itemize}

To realize these capabilities within a single framework, UniPro is trained in two stages. In Stage~1, the shared backbone is jointly optimized for semantic, in-context, and interactive segmentation, learning a unified representation under diverse conditioning forms. In Stage~2, the same backbone is further trained for volumetric propagation with propagation-specific supervision that improves slice-wise propagation reliability beyond per-slice losses. Detailed architectures and training procedures are provided in the Appendix~\ref{appendix_training}.

\begin{figure*}[t]
    \centering
    \includegraphics[width=\textwidth]{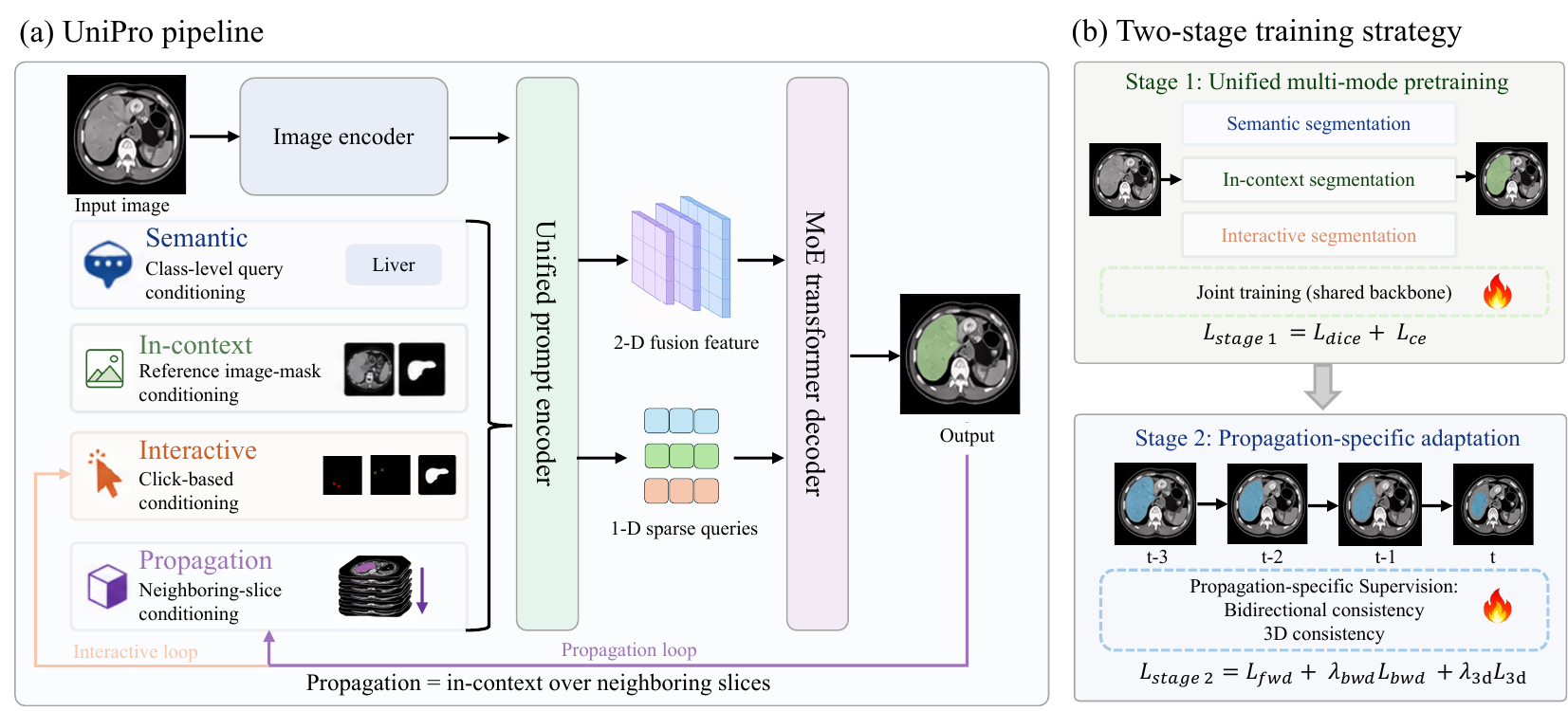}
    \caption{
    \textbf{UniPro overview.} \textbf{(a)} A unified slice-based segmentation pipeline for semantic, in-context, interactive, and propagation-based segmentation. UniPro uses a shared image encoder and decoder, with different settings instantiated through mode-specific conditioning inputs. Conditioning signals are converted into dense 2-D fusion features and sparse 1-D queries, which are jointly fed to the MoE Transformer decoder. Notably, propagation is formulated as in-context segmentation over neighboring slices, enabling coherent volumetric segmentation. \textbf{(b)} Two-stage training strategy. Stage 1 jointly pretrains semantic, in-context, and interactive segmentation with the shared backbone. Stage 2 further adapts the model to propagation with bidirectional and 3D consistency supervision.
    }
    \label{fig2}
\end{figure*}

\subsection{Propagation as Reference-Conditioned In-Context Segmentation}
\label{sec:propagation}

\paragraph{Anchor-pinned reference set construction.}
Given a 3D volume $\mathbf{V} \in \mathbb{R}^{D \times 3 \times H \times W}$ with $D$ axial slices, UniPro produces the full segmentation by bidirectionally propagating from a seed slice $z^*$, defined as the median axial index of the target organ. The accompanying seed mask $\hat{\mathbf{M}}_{z^*}$ may originate from semantic, in-context, or interactive segmentation, or from ground-truth annotation. For notational simplicity, we describe only the forward direction. At each step $q > z^*$, propagation is cast as a reference-conditioned in-context prediction problem, where the query slice $\mathbf{I}_q \in \mathbb{R}^{3 \times H \times W}$ is segmented using a dynamically maintained reference set
\begin{equation}
\mathcal{R}_q =
\{(\mathbf{I}_{z^*}, \hat{\mathbf{M}}_{z^*})\} \cup \{(\mathbf{I}_{q-j}, \hat{\mathbf{M}}_{q-j})\}_{j=1}^{r-1},
\end{equation}
with $r=5$ by default. The seed pair is permanently anchored in $\mathcal{R}_q$ and combined with the $r{-}1$ most recently propagated neighbors, providing a persistent, well-localized reference that mitigates cumulative semantic drift over long slice ranges. For brevity, we hereafter stack the reference images and masks in $\mathcal{R}_q$ as $\mathbf{I}_{\mathrm{ref}} \in \mathbb{R}^{r \times 3 \times H \times W}$ and $\mathbf{M}_{\mathrm{ref}} \in \mathbb{R}^{r \times 1 \times H \times W}$.

\paragraph{Dense reference-to-query feature matching.}
Given a query slice $\mathbf{I}_q$ together with its reference set $\mathcal{R}_q$, a shared image encoder extracts spatial key features from both sides:
\begin{equation}
\mathbf{K}^{\mathrm{query}} = \mathrm{Encoder}(\mathbf{I}_q) \in \mathbb{R}^{C \times hw}, \qquad
\mathbf{K}^{\mathrm{ref}} = \mathrm{Encoder}(\mathbf{I}_{\mathrm{ref}}) \in \mathbb{R}^{C \times rhw}.
\end{equation}
In parallel, a lightweight mask encoder embeds the reference image--mask pairs into context-aware value features:
\begin{equation}
\mathbf{V}^{\mathrm{ref}} =
\mathrm{MaskEncoder}\big(\mathrm{Concat}[\mathbf{I}_{\mathrm{ref}}, \mathbf{M}_{\mathrm{ref}}]\big)
\in \mathbb{R}^{C \times rhw}.
\end{equation}

Since $\mathbf{K}^{\mathrm{ref}}$ and $\mathbf{K}^{\mathrm{query}}$ reside in a common embedding space, pixel-wise correspondences between reference and query locations can be established directly. Formally, for a similarity function $c: \mathbb{R}^{C} \times \mathbb{R}^{C} \to \mathbb{R}$, we compute the affinity matrix and normalize it column-wise via softmax:
\begin{equation}
\mathbf{A}_{i,j} = c\!\left(\mathbf{K}^{\mathrm{ref}}_{:,i},\, \mathbf{K}^{\mathrm{query}}_{:,j}\right), \qquad
\mathbf{W}_{i,j} = \frac{\exp(\mathbf{A}_{i,j})}{\sum_n \exp(\mathbf{A}_{n,j})},
\end{equation}
where $\mathbf{A},\mathbf{W} \in \mathbb{R}^{rhw \times hw}$. We instantiate $c$ as a negative squared Euclidean similarity following~\citet{cheng2024putting} (see Appendix~\ref{app_incontext} for the explicit form). The normalized weights $\mathbf{W}$ then aggregate reference values into a query-aligned dense fusion feature carrying per-location semantic cues transferred from the reference set to the query:
\begin{equation}
\mathbf{F}_{\mathrm{fuse}} = \mathbf{V}^{\mathrm{ref}}\,\mathbf{W} \in \mathbb{R}^{C \times hw}.
\end{equation}


\paragraph{Sparse foreground--background queries.}
To complement the dense fusion feature with object-level context, we condense the reference key features $\mathbf{K}^{\mathrm{ref}}$ into a small set of 1-D foreground--background queries $\mathbf{Q}^{s} \in \mathbb{R}^{n_s \times C}$ via masked average pooling. Here, ``sparse'' refers to this compact query set relative to the spatially dense 2-D feature map: the queries summarize object-level semantics, whereas the dense features preserve spatially aligned, location-specific details. The first $n_s/2$ queries summarize foreground semantics and the remaining half summarize background semantics. Let $\mathbf{M}^{\mathrm{pool}} \in [0,1]^{rhw}$ denote $\mathbf{M}_{\mathrm{ref}}$ downsampled and flattened to the key feature resolution, so that $\mathbf{M}^{\mathrm{pool}}_u$ indicates whether spatial location $u$ is foreground. For the $n$-th query, we then define a binary selection weight $v_u^n \in \{0,1\}$ and compute
\begin{equation}
v_u^n =
\begin{cases}
1, & n \le \tfrac{n_s}{2} \ \text{and} \ \mathbf{M}^{\mathrm{pool}}_u \ge 0.5, \\[2pt]
1, & n > \tfrac{n_s}{2} \ \text{and} \ \mathbf{M}^{\mathrm{pool}}_u < 0.5, \\[2pt]
0, & \text{otherwise},
\end{cases}
\qquad
\mathbf{Q}^{s}_{n,:}
=
\frac{\sum_{u=1}^{rhw} (\mathbf{K}^{\mathrm{ref}})^\top_{u,:}\, v_u^n}
{\sum_{u=1}^{rhw} v_u^n} \in \mathbb{R}^{C}.
\end{equation}
The sparse foreground--background queries $\mathbf{Q}^{s}$ and dense fusion features $\mathbf{F}_{\mathrm{fuse}}$ jointly condition the decoder to predict the slice-wise mask:
\begin{equation}
\hat{\mathbf{M}}_q
=
\mathrm{Decoder}\!\left(\mathbf{I}_q \mid \mathbf{F}_{\mathrm{fuse}},\, \mathbf{Q}^{s}\right)
\in [0,1]^{H \times W}.
\end{equation}
Stacking the predictions along the axial dimension yields the full 3D segmentation $\hat{\mathbf{V}} \in [0,1]^{D \times H \times W}$. Notably, this formulation unifies slice-wise propagation and in-context segmentation: both operate on a query image paired with a set of image--mask exemplars and invoke the same dense-plus-sparse prompting pipeline. The two modes differ only in how $\mathcal{R}_q$ is composed: from recently segmented neighboring slices for propagation, or from externally annotated exemplars for in-context segmentation, while the prediction mechanism remains unchanged.

\subsection{Propagation and Consistency Objectives}
\label{sec:loss3d}

\paragraph{Stage II propagation objective.}
To train the model for volumetric propagation, we sample a short slice queue
$\mathcal{Q}=\{(\mathbf{I}_t, \mathbf{M}_t)\}_{t=1}^{T}$ of length $T=8$,
where $\mathbf{I}_t$ is the image at slice $t$ and $\mathbf{M}_t$ is the corresponding binary ground-truth mask. Beyond standard slice-wise supervision, our Stage II objective addresses two challenges specific to propagation: (1) error accumulation along the propagation direction, mitigated by backward propagation with seed reconstruction; and (2) lack of global volumetric consistency across slices, addressed by a lightweight 3D verification loss.

\paragraph{Forward propagation loss.}
Starting from the seed slice, the model obtains an initial mask $\hat{\mathbf{M}}_1$ and propagates the segmentation to subsequent slices. Each propagated slice is supervised by a standard 2D segmentation loss:
\begin{equation}
\mathcal{L}_{2D}(\hat{\mathbf{M}}_t, \mathbf{M}_t)
=
\mathcal{L}_{dice}(\hat{\mathbf{M}}_t, \mathbf{M}_t)
+
\mathcal{L}_{ce}(\hat{\mathbf{M}}_t, \mathbf{M}_t),
\end{equation}
where $\hat{\mathbf{M}}_t$ is the prediction at slice $t$.
The forward propagation loss aggregates the per-slice losses across the queue, with an additional weight $\lambda_{\text{init}}$ on the seed slice:
\begin{equation}
\mathcal{L}_{\text{fwd}}
=
\frac{1}{T}
\left(
\lambda_{\text{init}} \mathcal{L}_{2D}(\hat{\mathbf{M}}_1, \mathbf{M}_1)
+
\sum_{t=2}^{T} \mathcal{L}_{2D}(\hat{\mathbf{M}}_t, \mathbf{M}_t)
\right).
\end{equation}

\paragraph{Backward propagation and seed reconstruction.}
Forward-only propagation suffers from one-way cascading errors: a drift introduced at an early slice is silently inherited by all subsequent predictions with no corrective signal. We therefore additionally propagate from slice $T$ back to the seed, letting $\tilde{\mathbf{M}}_t$ denote the backward prediction at slice $t$; the reconstructed seed $\tilde{\mathbf{M}}_1$ then acts as a cycle-consistency checkpoint, whose supervision gradient penalizes forward drift at its source. We apply the same slice-wise supervision during backward propagation, with an additional weight $\lambda_{\text{seed}}$ on the reconstructed seed slice to emphasize cycle consistency:
\begin{equation}
\mathcal{L}_{\text{bwd}}
=
\frac{1}{T-2}
\sum_{t=2}^{T-1}
\mathcal{L}_{2D}(\tilde{\mathbf{M}}_{t}, \mathbf{M}_{t})
+
\lambda_{\text{seed}} \mathcal{L}_{2D}(\tilde{\mathbf{M}}_1, \mathbf{M}_1).
\end{equation}

The last slice is directly supervised by $\mathcal{L}_{\text{fwd}}$, while the first slice is explicitly reconstructed during backward propagation. The weight $\lambda_{\text{seed}}$ strengthens this reconstruction supervision to help limit accumulated propagation drift. Since our model is memory-based, forward and backward propagation invoke the same reference-conditioned function $f(\mathbf{I}_q, \mathcal{R}_q)$ and differ only in the content of $\mathcal{R}_q$, so $\mathcal{L}_{\text{bwd}}$ introduces no additional parameters. We further provide a theoretical analysis of $\mathcal{L}_{\text{bwd}}$ in Appendix~\ref{app:theory}.

\paragraph{3D supervision.}
While slice-wise supervision improves local accuracy, it does not explicitly enforce volumetric coherence: predictions can be locally plausible on each slice yet form an incoherent 3D shape overall. We therefore introduce a 3D verification loss that operates directly on stacked 2D predictions, without introducing 3D convolutions or other volumetric modules:
\begin{equation}
\hat{\mathbf{M}}^{3D} = [\hat{\mathbf{M}}_1, \hat{\mathbf{M}}_2, \dots, \hat{\mathbf{M}}_T],
\qquad
\mathbf{M}^{3D} = [\mathbf{M}_1, \mathbf{M}_2, \dots, \mathbf{M}_T].
\end{equation}
We then compare the predicted queue against the ground-truth queue using
\begin{equation}
\mathcal{L}_{\text{3D-ver}}
=
\mathcal{L}_{dice}^{3D}(\hat{\mathbf{M}}^{3D}, \mathbf{M}^{3D})
+
\mathcal{L}_{ce}^{3D}(\hat{\mathbf{M}}^{3D}, \mathbf{M}^{3D}),
\end{equation}
where $\mathcal{L}_{dice}^{3D}$ is computed over the entire queue jointly rather than averaging per-slice Dice scores, so errors concentrated in a few slices are not washed out by well-predicted ones. This term encourages coherent anatomical structure across the queue and discourages abrupt disappearance, leakage, or unrealistic shape variation between neighboring slices.

\paragraph{Overall Objective.}
To prevent the model from forgetting the native 2D datasets (e.g., X-ray, dermoscopy) seen in Stage I, we keep the Stage I objective active during Stage II training. The total Stage II loss combines both:
\begin{equation}
\mathcal{L}_{\text{stage II}}
=
\mathcal{L}_{\text{stage I}}
+
\mathcal{L}_{\text{fwd}}
+
\lambda_{\text{bwd}} \mathcal{L}_{\text{bwd}}
+
\lambda_{\text{3d}} \mathcal{L}_{\text{3d}}.
\end{equation}
This combined objective jointly improves local segmentation accuracy, propagation consistency in both directions, seed preservation after a forward-backward cycle, and global volumetric coherence across the full slice queue, while preserving the four-mode segmentation capability of the shared backbone. In our experiments, we set $\lambda_{\text{init}}=2$, $\lambda_{\text{seed}}=2$, $\lambda_{\text{bwd}}=0.7$, and $\lambda_{\text{3d}}=5$.


\section{Experiments}

\subsection{Experimental Setup}

\paragraph{Datasets.} We train UniPro on 9 publicly available datasets~\citep{Ma-2021-AbdomenCT-1K, CHAOS2021, CHAOSdata2019, bilic2023liver_lits, heller2019kits19, al2020dataset, jaeger2013automatic, candemir2013lung, ngoc2021neounet, codella2019skin, tschandl2018ham10000}, covering diverse imaging modalities, including CT, MRI, ultrasound, X-ray, and dermoscopy, as well as various clinical targets such as organs, tumors, and lesions. For evaluation, we additionally include four external datasets, BTCV~\citep{landman2015miccai}, ACDC~\citep{bernard2018acdc}, UW-SC~\citep{uwaterloo_skincancer}, and BUS~\citep{yap2017automated}, to assess cross-dataset generalization, and one unseen-class 3D dataset, BraTS~\citep{baid2021rsna}, to test adaptation to novel anatomical and pathological structures. Full dataset and preprocessing details are provided in Appendix~\ref{sec:appdataset}.

\paragraph{Baselines.}
We compare UniPro against three groups of baselines. For slice-based segmentation, we include semantic models nnU-Net (2D)~\citep{isensee2021nnu}, UniSeg~\citep{ye2023uniseg}, and Hermes~\citep{gao2024training}, and in-context models UniverSeg~\citep{butoi2023universeg}, SegMIC~\citep{zhao2026segmic}, and Iris~\citep{gao2025show}. For volumetric segmentation, we report nnU-Net (3D)~\citep{isensee2021nnu} as a fully supervised 3D upper-bound reference and nnInteractive~\citep{isensee2025nninteractive} as a strong interactive 3D reference. For propagation-based segmentation, we include SAM2~\citep{ravi2024sam}, PAM~\citep{chen2025pam}, MedSAM2~\citep{ma2025medsam2}, and iSegFormer~\citep{liu2022isegformer}. For UniPro itself, we report two inference modes: UniPro (no-prop), where each 2D slice is segmented independently under the corresponding semantic or in-context setting; and UniPro (prop), where only a seed slice is initialized by semantic, in-context or interactive segmentation and the remaining slices are segmented through the propagation mechanism, producing a coherent 3D segmentation from sparse 2D initialization.


\paragraph{Evaluation protocol.}
All slice-based baselines are trained on 2D slices extracted from our curated datasets under the same protocol as UniPro, ensuring a fair comparison at the slice level. With the exception of SAM2, the 3D and propagation baselines are trained on the same corresponding 3D datasets with matched preprocessing, optimization settings, and training budgets as UniPro, while preserving their method-specific input formats and inference interfaces. Performance is measured by the Dice coefficient: 2D Dice is computed per slice for native 2D datasets, and 3D Dice is computed over the entire volume formed by target-containing slices. Full implementation details are provided in the Appendix~\ref{appendix_training}.


\begin{table*}[t!]
\centering
\small
\caption{
Propagation-based segmentation (3D Dice, \%, $\uparrow$).
All methods are initialized from a single positive click on one seed slice and propagated to the remaining slices.
}
\label{tab1}
\renewcommand{\arraystretch}{0.95}
\setlength{\tabcolsep}{3pt}

\begin{tabular*}{\textwidth}{@{\extracolsep{\fill}}l|ccccc|cc|c}
\toprule
\multirow{2}{*}{\textbf{Method}}
& \multicolumn{5}{c|}{\textbf{Internal}}
& \multicolumn{2}{c|}{\textbf{External}}
& \textbf{Unseen} \\
\cmidrule(lr){2-6}\cmidrule(lr){7-8}\cmidrule(lr){9-9}
& Abd-1K
& CHAOS
& M\&Ms
& LiTS
& KiTS
& BTCV
& ACDC
& BraTS \\
\midrule

\multicolumn{9}{l}{\textit{3D Upper Bound}} \\
nnInteractive
& 76.38 & 83.04 & 78.58 & 49.76 & 63.08
& 67.44 & 79.87
& 51.25 \\
\midrule

\multicolumn{9}{l}{\textit{General Pretrained Model}} \\
SAM2
& 45.00 & 56.61 & 47.60 & 23.08 & 51.61
& 37.10 & 50.71
& 32.38 \\
\midrule

\multicolumn{9}{l}{\textit{Propagation Model}} \\
iSegFormer
& 83.00 & 82.27 & 82.55 & 51.46 & 68.77
& 64.78 & 77.05
& 42.25 \\

PAM
& 85.59 & 86.41 & 79.69 & 60.83 & 63.20
& 77.49 & 79.06
& 39.85 \\

MedSAM2
& 86.29 & \textbf{90.64} & 82.97 & 62.25 & 73.32
& 80.55 & 84.18
& 40.74 \\

UniPro (prop)
& \textbf{88.62} & 89.70 & \textbf{87.64} & \textbf{66.95} & \textbf{76.95}
& \textbf{84.40} & \textbf{87.97}
& \textbf{54.89} \\

\bottomrule
\end{tabular*}
\vspace{-0.8em}
\end{table*}

\subsection{Main Results on Interactive-to-Propagation Segmentation}

Table~\ref{tab1} evaluates UniPro under a single-click propagation setting on 3D datasets, reflecting a practical annotation workflow where a user provides an initial click on one representative slice and the model completes the full volume automatically. UniPro (prop) achieves the best performance on 7 of 8 datasets, improving over the strongest propagation baseline MedSAM2 by 4.70 Dice points on LiTS (66.95\% vs 62.25\%) and 4.67 Dice points on M\&Ms, and surpassing PAM and iSegFormer by larger margins on nearly all datasets. SAM2, despite large-scale natural image pretraining, performs poorly on medical volumes, indicating that natural-image pretraining alone does not transfer to volumetric medical segmentation. Notably, UniPro (prop) also outperforms the dedicated 3D interactive model nnInteractive, suggesting that a well-trained slice-based propagation can rival full-3D interactive approaches. UniPro further generalizes beyond the training distribution. On external datasets, it reaches 84.40\% on BTCV and 87.97\% on ACDC. Most strikingly, on BraTS (brain tumors, an anatomy unseen during training), UniPro achieves 54.89\% compared to MedSAM2's 40.74\%, a gain of 14.15 Dice points that suggests the reference-conditioned mechanism transfers beyond the anatomies seen during training. Together, these results demonstrate that UniPro delivers both reliable seed-slice interaction and robust volumetric propagation across domain shifts and novel anatomies, supporting accurate 3D segmentation from minimal user input. More statistical analysis of these results is provided
in Appendix~\ref{app:propagation_statistics}.

\subsection{Main Results on Semantic and In-Context Segmentation}

Table~\ref{tab2} reports semantic and in-context segmentation across 3D and 2D datasets. On 3D datasets, UniPro (prop) achieves the best results on nearly all internal datasets in both semantic and in-context settings. In the semantic setting, UniPro (prop) reaches 87.73\% on M\&Ms and 89.49\% on ACDC, surpassing the fully supervised 3D nnU-Net upper bound on both datasets despite using a 2D-based architecture. Propagation consistently improves over UniPro (no-prop) on six of the seven 3D datasets (e.g., +2.89 Dice points on CHAOS, +2.45 on KiTS), demonstrating that single-slice initialization combined with propagation effectively leverages volumetric continuity. Similar gains appear in the in-context setting, where UniPro (prop) outperforms strong baselines (UniverSeg, SegMIC, Iris) on most 3D datasets, including a gain of 2.42 Dice points over Iris on Abd-1K (87.64\% vs 85.22\%) and a gain of 4.93 Dice points on ACDC (88.44\% vs 83.51\%). This confirms that propagation serves as an effective intra-volume extension of in-context segmentation. On 2D datasets, UniPro remains competitive but does not always achieve the best performance: while it leads on UW-SC, specialized 2D models such as Hermes and Iris achieve the best results on most internal datasets. This reflects a training trade-off introduced by Stage-2 propagation adaptation: as the model is further trained for sequential slice prediction, it slightly drifts from the optimal solution for static 2D segmentation. We further discuss this trade-off in Discussion.

\begin{table*}[t!]
\centering
\small
\caption{
Semantic and in-context segmentation across internal and external datasets (Dice, \%, $\uparrow$). All in-context models use one-shot inference; propagation is applied only to 3D volumes. nnU-Net (3D) is reported as a fully supervised upper bound. Best results in each block are in bold.
}
\label{tab2}
\setlength{\tabcolsep}{2.2pt}
\renewcommand{\arraystretch}{0.9}

\begin{tabular*}{\textwidth}{@{\extracolsep{\fill}}l|ccccc|cc|cccc|cc}
\toprule
\multirow{3}{*}{\textbf{Method}}
& \multicolumn{7}{c|}{\textbf{3D Datasets}}
& \multicolumn{6}{c}{\textbf{2D Datasets}} \\
\cmidrule(lr){2-8}\cmidrule(lr){9-14}
& \multicolumn{5}{c|}{\textbf{Internal}}
& \multicolumn{2}{c|}{\textbf{External}}
& \multicolumn{4}{c|}{\textbf{Internal}}
& \multicolumn{2}{c}{\textbf{External}} \\
\cmidrule(lr){2-6}\cmidrule(lr){7-8}
\cmidrule(lr){9-12}\cmidrule(lr){13-14}
& \shortstack{{\scriptsize Abd-1K}}
& \shortstack{{\scriptsize CHAOS}}
& \shortstack{{\scriptsize M\&Ms}}
& \shortstack{{\scriptsize LiTS}}
& \shortstack{{\scriptsize KiTS}}
& \shortstack{{\scriptsize BTCV}}
& \shortstack{{\scriptsize ACDC}}
& \shortstack{{\scriptsize Breast}}
& \shortstack{{\scriptsize BKAI}}
& \shortstack{{\scriptsize Chest}}
& \shortstack{{\scriptsize ISIC}}
& \shortstack{{\scriptsize BUS}}
& \shortstack{{\scriptsize UW-SC}} \\
\midrule

\multicolumn{14}{l}{\textit{3D Upper Bound}} \\
nnU-Net (3D)
& 88.45 & 90.15 & 85.47 & 63.14 & 81.28
& 86.02 & 88.59
& -- & -- & -- & --
& -- & -- \\
\midrule

\multicolumn{14}{l}{\textit{Semantic segmentation}} \\

nnU-Net (2D)
& 79.13 & 80.29 & 85.42 & 60.12 & 75.99
& 70.73 & 84.64
& 71.68 & 83.45 & 94.95 & \textbf{86.76}
& 68.00 & 79.78 \\

UniSeg
& 84.26 & 84.17 & 85.66 & 56.23 & 72.68
& 81.53 & 86.91
& 72.32 & 81.33 & 94.51 & 83.86
& 67.12 & 77.93 \\

Hermes
& 86.43 & 86.25 & 86.15 & 60.62 & 75.17
& 84.21 & 86.59
& \textbf{75.21} & \textbf{84.12} & \textbf{95.59} & 85.33
& \textbf{68.67} & 82.31 \\

UniPro (no-prop)
& 87.30 & 85.01 & 86.95 & 60.85 & 75.56
& \textbf{84.36} & 87.60
& 73.87 & 82.72 & 95.03 & 86.11
& 67.29 & \textbf{82.57} \\

UniPro (prop)
& \textbf{88.20} & \textbf{87.90} & \textbf{87.73} & \textbf{61.15} & \textbf{78.01}
& 83.59 & \textbf{89.49}
& -- & -- & -- & --
& -- & -- \\
\midrule

\multicolumn{14}{l}{\textit{In-context segmentation}} \\

UniverSeg
& 64.21 & 67.61 & 63.38 & 45.42 & 60.70
& 58.93 & 59.81
& 56.90 & 73.91 & 94.88 & 79.60
& 45.23 & 68.31 \\

SegMIC
& 74.12 & 72.60 & 82.91 & 48.35 & 61.56
& 72.77 & 78.14
& 64.18 & 77.53 & \textbf{96.13} & 81.20
& 68.39 & 80.60 \\

Iris
& 85.22 & 84.31 & 84.77 & 59.10 & 74.29
& 79.94 & 83.51
& \textbf{73.99} & \textbf{82.53} & 95.40 & \textbf{85.21}
& \textbf{71.44} & 82.57 \\

UniPro (no-prop)
& 87.01 & 84.48 & 86.65 & 57.77 & 75.65
& 82.48 & 85.91
& 71.63 & 80.97 & 94.99 & 84.62
& 68.78 & \textbf{84.92} \\

UniPro (prop)
& \textbf{87.64} & \textbf{85.01} & \textbf{87.31} & \textbf{61.15} & \textbf{76.37}
& \textbf{83.18} & \textbf{88.44}
& -- & -- & -- & --
& -- & -- \\
\bottomrule
\end{tabular*}
\vspace{-0.8em}
\end{table*}

\begin{table}[t]
\centering
\caption{
Ablation studies on Stage-2 propagation training.
}
\label{tab:ablation_combined}
\begin{subtable}[t]{0.48\linewidth}
\centering
\caption{Effect of bidirectional and 3D verification supervision on ACDC.}
\label{tab:ablation_bidir_3dloss_acdc}
\resizebox{\linewidth}{!}{
\begin{tabular}{lccc}
\toprule
Variant 
& Semantic
& In-context
& Interactive \\
\midrule
w/o $\mathcal{L}_{\text{bwd}}$ \& $\mathcal{L}_{\text{3D}}$
& 87.78 
& 86.13
& 85.61 \\
w/o $\mathcal{L}_{\text{bwd}}$
& 89.10
& 87.39 
& 87.28 \\
w/o $\mathcal{L}_{\text{3D}}$
&  87.92
&  86.51
& 86.20 \\
Full model
& 89.49 
& 88.44 
& 87.97 \\
\bottomrule
\end{tabular}
}
\end{subtable}
\hfill
\begin{subtable}[t]{0.48\linewidth}
\centering
\caption{Trade-off between 3D propagation gains and native 2D segmentation accuracy.}
\label{tab:stage2_interleaving}
\resizebox{\linewidth}{!}{
\begin{tabular}{lcc}
\toprule
Training Strategy 
& ACDC 
& ISIC \\
\midrule
Stage-1 only 
& 86.73
& 87.54 \\
Stage-2 w/o 2D interleaving
& 89.46 
& 82.40 \\
Stage-2 w/ 2D interleaving
& 89.49 
& 86.11 \\
\bottomrule
\end{tabular}
}
\end{subtable}
\vspace{-0.8em}
\end{table}

\subsection{Ablation Study}

\paragraph{Effect of propagation supervision.}
Table~\ref{tab:ablation_combined}(a) ablates the bidirectional consistency loss $\mathcal{L}_{\text{bwd}}$ and the 3D verification loss $\mathcal{L}_{\text{3D}}$ on ACDC under three propagation modes. Removing both reduces semantic Dice from 89.49\% to 87.78\%, with comparable drops on in-context (88.44\% to 86.13\%) and interactive (87.97\% to 85.61\%), confirming that propagation-specific supervision is essential beyond per-slice losses. Individually, $\mathcal{L}_{\text{3D}}$ contributes the larger share (a 1.57-Dice-point drop in semantic segmentation when removed) by enforcing cross-slice coherence that per-slice optimization cannot capture, while $\mathcal{L}_{\text{bwd}}$ adds a smaller but consistent gain (0.39 Dice points on top of $\mathcal{L}_{\text{3D}}$) by suppressing one-way drift. The two are complementary: $\mathcal{L}_{\text{3D}}$ regularizes global volumetric structure, $\mathcal{L}_{\text{bwd}}$ stabilizes the sequential propagation process.

\paragraph{Effect of interleaved Stage-2 training.}
We study the trade-off between learning 3D propagation and preserving native 2D segmentation ability. As shown in Table~\ref{tab:ablation_combined}(b), Stage-2 training substantially improves ACDC (86.73\% to 89.46\%) but degrades ISIC by 5.14 Dice points (87.54\% to 82.40\%) when trained only on propagation samples, indicating partial forgetting of 2D capability. Interleaving 2D slice-level batches during Stage-2 recovers most of the lost 2D performance (ISIC 82.40\% to 86.11\%) while leaving ACDC essentially unchanged (89.46\% to 89.49\%). However, ISIC remains 1.43 Dice points below the Stage-1 only model, suggesting an inherent trade-off between volumetric propagation and native 2D segmentation. A complementary ablation further supports the benefit of combining Stage-I pretraining with Stage-II propagation training: the full pipeline outperforms either stage alone under GT-seed initialization (Appendix~\ref{app:training_stage_ablation}).

\paragraph{Effect of flexible interactive correction.}
We evaluate whether propagation reduces annotation effort during interactive 3D segmentation. Starting from an empty mask, the user iteratively places clicks until the 3D Dice reaches a target threshold ($T=0.90$ for ACDC, $T=0.80$ for the unseen BraTS) or a click budget of 100 is exhausted. We compare two settings: a slice-wise baseline, where each click updates only the clicked slice, and propagation, where the slice refined by user clicks is added to an anchor pool and the volume is re-segmented through propagation from the updated anchors. As shown in Figure~\ref{fig:click_efficiency_sub}, propagation reaches the target Dice with fewer clicks (27.7 vs 25.2 on ACDC; 92.0 vs 57.0 on BraTS) and substantially fewer edited slices (8.5 vs 3.3 on ACDC; 62.4 vs 12.1 on BraTS). The gain is most pronounced on BraTS, where fragmented tumor regions make slice-by-slice annotation inefficient: propagation reduces the number of edited slices by over 80\% (62.4 to 12.1) and clicks by nearly 40\% (92.0 to 57.0). These results show that corrected slices serve as effective propagation anchors, reducing the numbers of clicks and edited slices in the simulated annotation protocol. Full simulation details are provided in Appendix~\ref{appendix:interactive-vs-propagate}.

\begin{wrapfigure}{r}{0.48\textwidth}
\vspace{-1.0em}
\centering
\begin{subfigure}[t]{\linewidth}
\centering
\includegraphics[width=\linewidth]{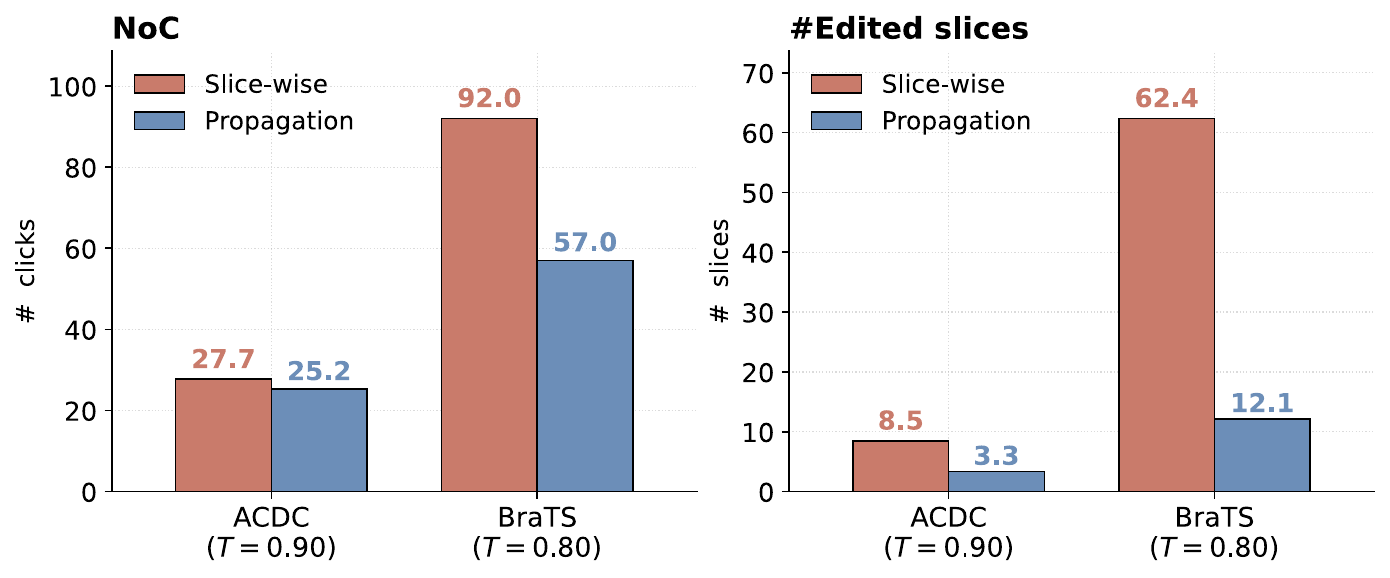}
\caption{Annotation efficiency under interactive correction: UniPro reaches the target 3D Dice with fewer clicks and edited slices than slice-wise editing.}
\label{fig:click_efficiency_sub}
\end{subfigure}

\begin{subfigure}[t]{0.92\linewidth}
\centering
\includegraphics[width=\linewidth]{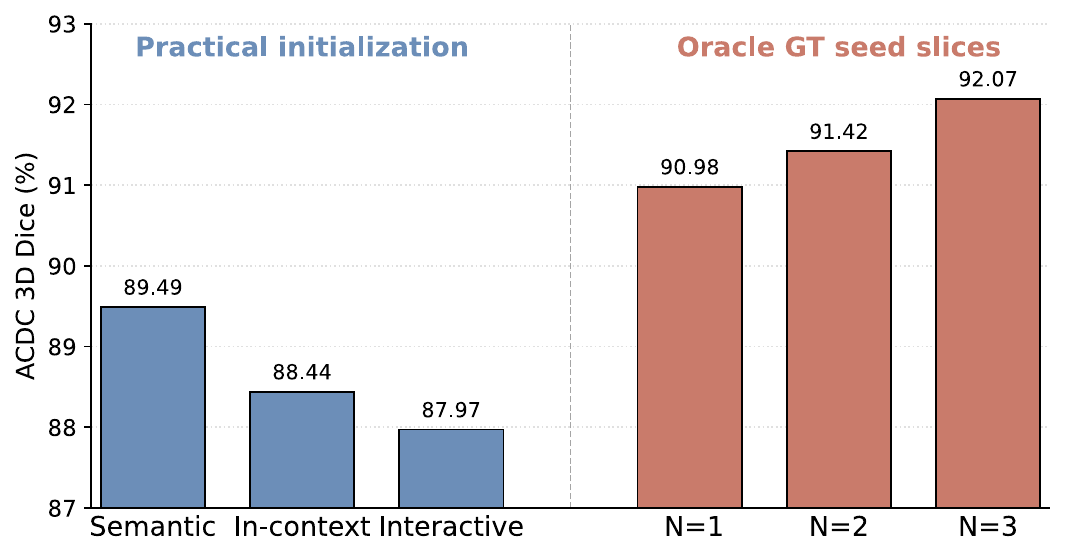}
\caption{Seed initialization on ACDC: practical seeds (semantic, in-context, interactive) yield competitive single-seed propagation; multiple GT reference slices further improve performance.}
\label{fig:seed_initialization_sub}
\end{subfigure}

\caption{Ablation analysis of UniPro's 3D propagation behavior.}
\label{fig:propagation_ablation}
\vspace{-1.2em}
\end{wrapfigure}

\paragraph{Effect of seed initialization.}
We analyze how reference quality and quantity affect volumetric propagation on ACDC. Here $N$ denotes the number of GT slices retained as anchors in the sliding reference window during propagation. As shown in Figure~\ref{fig:seed_initialization_sub}, practical seeds (semantic, in-context, interactive) reach 87.97--89.49\% Dice from a single seed slice, only 1--3 Dice points below the oracle GT single-seed setting (90.98\%). Increasing the number of GT reference slices monotonically improves performance, from 90.98\% at $N=1$ to 92.07\% at $N=3$, indicating that UniPro can effectively exploit additional high-quality anchors when available. Together, these results show that propagation is robust to the choice of practical initialization and scales gracefully with stronger references.

\paragraph{Robustness to seed-slice selection.}
On ACDC, we vary the seed location across five quantiles of the target extent and a randomly selected target-containing slice, using ground-truth seed masks to isolate the effect of location. Dice remains within 89.46--91.03\%, with random initialization achieving 90.12\% compared with 90.98\% for the median seed. Central slices generally yield slightly better results, consistent with their typically larger target regions providing more informative references, whereas boundary seeds show only modest degradation. Overall, propagation remains effective across the evaluated seed locations under GT-mask initialization; detailed results are provided in Appendix~\ref{app:seed_robustness}.





\subsection{Visualization}
\label{sec:visualization}

We complement the quantitative evaluation with qualitative
comparisons and success--failure analysis.
In the pancreas example
(Figure~\ref{fig:pancreas_qualitative}),
UniPro maintains coherent segmentation across distant
slices, whereas the compared baselines exhibit leakage,
incomplete coverage, or loss of the target.
Figure~\ref{fig:qualitative_multicomponent} extends this
analysis to irregular and multi-component tumors in LiTS,
KiTS, and BraTS, showing the seed slice, selected propagated
slices at increasing absolute distances, and corresponding
3D ground-truth and predicted masks.
In the successful examples, UniPro segments the dominant
irregular tumors in LiTS and KiTS and a relatively compact
lesion in BraTS, despite brain anatomy and brain tumors
being absent from training.
Failures include missed small or spatially separated
components in LiTS and KiTS and incomplete segmentation
of a heterogeneous, fragmented BraTS lesion.
Local errors also remain in the successful examples.
These observations suggest that propagation depends
not only on geometric complexity or anatomical familiarity,
but also on component size, appearance consistency,
and cross-slice continuity.
Detailed visualizations and analyses are provided
in Appendix~E.

\section{Discussion}

We have presented UniPro, a unified framework that connects multiple segmentation paradigms across native 2D and 3D data. UniPro is particularly suited to heterogeneous segmentation and volumetric annotation workflows that combine automatic or reference-guided initialization, interactive correction, and propagation. Its shared interface allows difficult slices to be corrected locally and the corrections propagated through the volume. Dedicated models may remain preferable for fixed single-task applications.

\paragraph{Limitations and future work.}Several limitations remain and motivate future work. \textbf{Trade-offs between 2D and 3D.} Stage-2 propagation adaptation improves 3D segmentation but can reduce native 2D accuracy, which interleaved 2D training only partly mitigates. Compared with the 3D nnU-Net baseline, UniPro uses less peak inference memory but requires longer inference time, partly due to sequential propagation (Appendix~\ref{app:runtime_memory}). Future work could explore better multi-task optimization and propagation strategies that allow greater parallelization. \textbf{Gap to fully supervised 3D models.} UniPro trails the fully supervised 3D nnU-Net on most datasets in the semantic setting. It processes ordered 2D slices without explicit spacing or physical-distance inputs. Although isotropic spacing is not required, large through-plane gaps can weaken cross-slice correspondence. Incorporating spatial information, stronger cross-slice integration, or lightweight volumetric components may help close this gap. \textbf{Irregular and multi-component targets.}
UniPro imposes no explicit convexity or single-component
constraint. However, single-click initialization may miss
disconnected regions, and weak cross-slice correspondence
can hinder propagation. Additional interaction or alternative
initialization may improve coverage. Robustness to fragmented
targets and evaluation on untested structures, such as the
intestine, remain directions for future work. \textbf{Towards a medical segmentation foundation model.} Our training corpus remains modest relative to that of large-scale models such as SAM2~\citep{ravi2024sam}. Expanding data diversity and model capacity, together with broader evaluation across domains, is an important direction for improving generalization while retaining the unified segmentation interface.

\section{Conclusion}
We present UniPro, a unified slice-based framework that supports semantic, in-context, interactive, and propagation-based medical image segmentation within a single model. By formulating volumetric propagation as a special case of reference-conditioned in-context prediction, UniPro bridges native 2D images and 3D volumes through a shared backbone, enabling coherent volumetric segmentation from sparse 2D initialization. Extensive experiments across diverse modalities, anatomies, and segmentation settings show that UniPro achieves strong performance and cross-dataset generalization, while reducing clicks and edited slices in simulated interactive volumetric annotation. We believe this work points to a promising direction toward scalable, unified medical segmentation models that generalize across tasks, modalities, and dimensions.

\clearpage

\section*{Acknowledgments}
This research has been partially funded by research grants to D.~Metaxas through NSF: 2310966, 2235405, 2212301, 2003874, 1951890 and NIH 2R01HL127661.

\paragraph{Competing interests.}
The authors declare no competing interests.

\bibliographystyle{plainnat}
\bibliography{main}

\appendix

\section{Method}

\subsection{Problem Definition}

\subsubsection{Unified multi-mode segmentation.}
UniPro formulates medical image segmentation as a unified condition-driven prediction problem, where different forms of guidance specify what or where to segment within the same shared architecture. Given an input image $\mathbf{I}$ and a mode-specific condition $\mathcal{G}$, the model predicts
\begin{equation}
    \hat{\mathbf{M}} = f_{\theta}(\mathbf{I} \mid \mathcal{G}).
\end{equation}
Under this formulation, UniPro supports three primary 2D segmentation modes (semantic, in-context, and interactive) and further extends in-context segmentation to volumetric propagation as a fourth mode.
\begin{itemize}
    \item \textbf{Mode-1: Semantic segmentation.}
    Semantic segmentation uses learned class-level priors as the condition. Given a learnable class embedding matrix $\mathbf{P} \in \mathbb{R}^{N_{\mathrm{cls}} \times (p \times C)}$, each class $n$ is represented by $\mathbf{p}_n \in \mathbb{R}^{p \times C}$, which contains $p$ semantic query vectors. The prediction for class $n$ is formulated as
    \begin{equation}
        \hat{\mathbf{M}}_n = f_{\theta}(\mathbf{I} \mid \mathbf{p}_n).
    \end{equation}
    \item \textbf{Mode-2: In-context segmentation.}
    In-context segmentation defines the target object using annotated reference examples instead of a predefined class label. Given a query image $\mathbf{I}_q$ and a support set $\mathcal{S}=\{(\mathbf{I}_{\mathrm{ref}}^{(i)}, \mathbf{M}_{\mathrm{ref}}^{(i)})\}_{i=1}^{K}$, UniPro predicts
    \begin{equation}
        \hat{\mathbf{M}}_q = f_{\theta}(\mathbf{I}_q \mid \mathcal{S}).
    \end{equation}
    \item \textbf{Mode-3: Interactive segmentation.}
    Interactive segmentation uses user-provided prompts, such as positive and negative clicks, as the condition. Given a click set $\mathcal{C}=\{c_i\}_{i=1}^{N_c}$, the model predicts
    \begin{equation}
        \hat{\mathbf{M}} = f_{\theta}(\mathbf{I} \mid \mathcal{C}).
    \end{equation}
    This mode can either produce masks from scratch or refine predictions initialized by semantic or in-context segmentation.
    \item \textbf{Mode-4: Propagation-based volumetric segmentation.}
    Propagation is treated as an intra-volume special case of in-context segmentation. Instead of using external annotated support images, the support set is constructed from previously segmented slices within the same 3D volume. Given a volume $\mathbf{V}=\{\mathbf{I}_z\}_{z=1}^{D}$ and a seed prediction $\hat{\mathbf{M}}_{z^*}$ initialized by any of the 2D modes or provided as ground truth, UniPro segments the remaining slices by repeatedly applying
    \begin{equation}
        \hat{\mathbf{M}}_z = f_{\theta}(\mathbf{I}_z \mid \mathcal{R}_z),
    \end{equation}
    where $\mathcal{R}_z=\{(\mathbf{I}_{r}, \hat{\mathbf{M}}_{r})\}_{r \in \mathcal{N}(z)}$ denotes the intra-volume reference set formed by the seed slice and nearby propagated slices. Thus, Mode-4 shares the same reference-conditioned formulation as Mode-2, but restricts the support examples to the current volume and applies the prediction sequentially along the slice dimension. This enables UniPro to handle 3D volumetric segmentation without introducing a separate 3D-specific architecture.
\end{itemize}
In the following sections, we describe how UniPro realizes these modes within a shared architecture by converting heterogeneous segmentation conditions into a unified condition representation.

\subsubsection{Semantic segmentation}
Given an input image $\mathbf{I} \in \mathbb{R}^{3 \times H \times W}$, the image encoder extracts a visual feature map:
\begin{equation}
    \mathbf{F} = \mathrm{Encoder}(\mathbf{I}), 
    \quad \mathbf{F} \in \mathbb{R}^{C \times h \times w},
\end{equation}
where $C$ denotes the feature dimension and $h, w$ denote the spatial resolution of the encoded features.
In semantic segmentation, the condition is provided by learned class-level sparse queries; since no spatial reference mask or user interaction is available, this mode relies solely on sparse queries. We define a learnable semantic prompt matrix $\mathbf{P} \in \mathbb{R}^{N_{\mathrm{cls}} \times (p \times C)}$, where each class $n$ is associated with a prompt set $\mathbf{p}_n \in \mathbb{R}^{p \times C}$. These prompts are optimized during training and encode class-level anatomical priors. To segment class $n$, UniPro feeds the image feature $\mathbf{F}$ and the corresponding semantic prompt $\mathbf{p}_n$ into the shared decoder:
\begin{equation}
    \hat{\mathbf{M}}_n = \mathrm{Decoder}(\mathbf{F} \mid \mathbf{p}_n).
\end{equation}

\subsubsection{In-context segmentation}
\label{app_incontext}
In-context segmentation uses annotated reference examples as the conditioning information. Let the support set be $\mathcal{S} = \{(\mathbf{I}_{\mathrm{ref}}, \mathbf{M}_{\mathrm{ref}})\}$, where 
$\mathbf{I}_{\mathrm{ref}} \in \mathbb{R}^{N_{\mathrm{ref}} \times 3 \times H \times W}$ 
denotes the reference images and 
$\mathbf{M}_{\mathrm{ref}} \in \mathbb{R}^{N_{\mathrm{ref}} \times 1 \times H \times W}$ 
denotes the corresponding segmentation masks. This conditioning signal is encoded into two complementary components: sparse foreground--background queries and spatially aligned dense fusion features.

Given a query image $\mathbf{I}_q$ and reference images $\mathbf{I}_{\mathrm{ref}}$, the shared image encoder first extracts query and reference features:
\begin{equation}
\mathbf{F}_q = \mathrm{Encoder}(\mathbf{I}_{q}), 
\quad
\mathbf{F}_k = \mathrm{Encoder}(\mathbf{I}_{\mathrm{ref}}),
\end{equation}
while the reference image--mask pairs are encoded into value features by a lightweight mask encoder:
\begin{equation}
\mathbf{F}_v = \mathrm{MaskEncoder}(\mathrm{Concat}[\mathbf{I}_{\mathrm{ref}}, \mathbf{M}_{\mathrm{ref}}]).
\end{equation}

For dense prompts, these encoded features are flattened as
\begin{equation}
\mathbf{K}^{\mathrm{ref}} = \mathrm{Flat}(\mathbf{F}_k) \in \mathbb{R}^{C \times N_{\mathrm{ref}}hw}, \quad
\mathbf{V}^{\mathrm{ref}} = \mathrm{Flat}(\mathbf{F}_v) \in \mathbb{R}^{C \times N_{\mathrm{ref}}hw}, \quad
\mathbf{K}^{q} = \mathrm{Flat}(\mathbf{F}_q) \in \mathbb{R}^{C \times hw}.
\end{equation}

Since $\mathbf{K}^{\mathrm{ref}}$ and $\mathbf{K}^{q}$ are extracted by the same encoder, they lie in a shared feature space where query--reference correspondences can be computed. The value features $\mathbf{V}^{\mathrm{ref}}$, derived from reference image--mask pairs, are then transferred to the query space through these correspondences, providing dense mask-aware guidance. Following~\citep{cheng2021rethinking, cheng2024putting}, we adopt the negative squared Euclidean distance to compute the affinity matrix $\mathbf{A}$:
\begin{equation}
\mathbf{A}_{i,j} = -\left\| \mathbf{K}^{\mathrm{ref}}_{:,i} - \mathbf{K}^{q}_{:,j} \right\|_2^2 
= 2\, \mathbf{K}^{\mathrm{ref}}_{:,i}{}^{\top}\mathbf{K}^{q}_{:,j} 
- \left\| \mathbf{K}^{\mathrm{ref}}_{:,i} \right\|_2^2 
- \left\| \mathbf{K}^{q}_{:,j} \right\|_2^2,
\end{equation}
where the last term can be omitted as shown in~\citep{cheng2021rethinking}, which improves efficiency and reduces computational cost. The affinity matrix is then normalized along the reference dimension to obtain the soft assignment weights:
\begin{equation}
\mathbf{W}_{i,j} = \frac{\exp(\mathbf{A}_{i,j})}{\sum_{n} \exp(\mathbf{A}_{n,j})}, 
\quad
\mathbf{A}, \mathbf{W} \in \mathbb{R}^{N_{\mathrm{ref}}hw \times hw}.
\end{equation}
The affinity weights $\mathbf{W}$ project the value features $\mathbf{V}^{\mathrm{ref}}$ into the query-aligned space, producing the 2D fusion feature map:
\begin{equation}
\mathbf{F}_{\mathrm{fuse}} = \mathbf{V}^{\mathrm{ref}}\mathbf{W}, 
\quad
\mathbf{F}_{\mathrm{fuse}} \in \mathbb{R}^{C \times hw}.
\end{equation}
To facilitate subsequent operations, feature channels at all scales are projected to 256 via a linear layer before being fed into the decoder.

For the sparse foreground--background queries, we construct a set of $n_s$ 1-D query vectors $\mathbf{Q}^s \in \mathbb{R}^{n_s \times C}$ from the encoded reference features $\mathbf{K}^{\mathrm{ref}}$, where the first half represents foreground queries and the second half represents background queries. A pooling mask $\mathbf{M}^{\mathrm{pool}} \in [0,1]^{N_{\mathrm{ref}}hw}$ is obtained by downsampling and flattening $\mathbf{M}_{\mathrm{ref}}$. The binary selection vector for the $n$-th object query is defined as
\begin{equation}
\mathbf{v}^{n}_i =
\begin{cases}
0, & n \le \frac{n_s}{2} \ \text{and} \ \mathbf{M}^{\mathrm{pool}}_i < 0.5, \\[4pt]
0, & n > \frac{n_s}{2} \ \text{and} \ \mathbf{M}^{\mathrm{pool}}_i \ge 0.5, \\[4pt]
1, & \text{otherwise}.
\end{cases}
\end{equation}
The corresponding object query is computed by masked average pooling:
\begin{equation}
\mathbf{Q}^{s}_{n,:} =
\frac{\sum_{i=1}^{N_{\mathrm{ref}}hw}(\mathbf{K}^{\mathrm{ref}})^{\top}_{i,:}\mathbf{v}^{n}_i}
{\sum_{i=1}^{N_{\mathrm{ref}}hw}\mathbf{v}^{n}_i},
\quad
\mathbf{Q}^{s}_{n,:}\in \mathbb{R}^{C}.
\end{equation}

In summary, the reference examples act as the condition for in-context segmentation. They are encoded into sparse foreground--background queries that capture object-level semantics and dense fusion features that provide spatially aligned mask guidance. The shared decoder then predicts the query mask conditioned on both components:
\begin{equation}
\hat{\mathbf{M}}_q =
\mathrm{Decoder}(\mathbf{F}_{\mathrm{fuse}} \mid \mathbf{Q}^{s}).
\end{equation}

\subsubsection{Interactive segmentation}
Interactive segmentation conditions the prediction on user-provided clicks. Given an input image $\mathbf{I}$ and a click set $\mathcal{C}=\{c_i\}_{i=1}^{N_c}$, we follow~\citep{sofiiuk2022reviving, liu2023simpleclick} and assemble a three-channel click prompt $\mathbf{I}_{\mathrm{click}} \in \mathbb{R}^{3 \times H \times W}$: the first two channels are disk-encoded maps for positive and negative clicks, and the third channel carries the mask predicted in the previous interaction round, which serves as a memory of past corrections. This prompt provides a 2D dense signal for spatial refinement and is processed by a lightweight prompt encoder.

The encoded click feature is then fused with the visual feature, where the choice of visual feature depends on whether interactive segmentation operates standalone or as a refinement step on top of another mode:
\begin{equation}
\mathbf{F}_{\mathrm{fuse}} =
\begin{cases}
\mathrm{PromptEncoder}(\mathbf{I}_{\mathrm{click}}) + \mathrm{Encoder}(\mathbf{I}), 
& \text{Mode-3 or refining Mode-1}, \\[4pt]
\mathrm{PromptEncoder}(\mathbf{I}_{\mathrm{click}}) + \mathbf{V}^{\mathrm{ref}}\mathbf{W}, 
& \text{refining Mode-2},
\end{cases}
\quad
\mathbf{F}_{\mathrm{fuse}} \in \mathbb{R}^{C \times hw}.
\end{equation}
When the model runs as a standalone interactive segmenter (Mode-3) or refines a semantic prediction (Mode-1), the click feature is added to the image feature directly. When refining an in-context prediction (Mode-2), it is instead added to the query-aligned reference feature $\mathbf{V}^{\mathrm{ref}}\mathbf{W}$, so that user clicks correct the reference-conditioned output rather than the raw image feature.

In parallel, each click is converted into a sparse object query. Following~\citep{guo2025verse}, a click at image coordinates $(x,y)$ is mapped to its feature-grid location $(x',y')=(\lfloor x/s \rfloor,\lfloor y/s \rfloor)$ with downsampling ratio $s$. We then average-pool features within a $(2r+1)\times(2r+1)$ window centered at $(x',y')$, apply a linear projection, and add a SAM-style positional embedding~\citep{kirillov2023segment}. Stacking the results over all clicks yields the click query set $\mathbf{Q}^{c}\in\mathbb{R}^{N_c \times C}$.

The shared decoder consumes the fused feature together with the appropriate sparse queries, where the queries depend on which mode is being refined:
\begin{equation}
\hat{\mathbf{M}} =
\begin{cases}
\mathrm{Decoder}(\mathbf{F}_{\mathrm{fuse}} \mid \mathbf{p}_n, \mathbf{Q}^{c}), 
& \text{refining Mode-1}, \\[4pt]
\mathrm{Decoder}(\mathbf{F}_{\mathrm{fuse}} \mid \mathbf{Q}^{s}, \mathbf{Q}^{c}), 
& \text{refining Mode-2}, \\[4pt]
\mathrm{Decoder}(\mathbf{F}_{\mathrm{fuse}} \mid \mathbf{Q}^{c}), 
& \text{Mode-3}.
\end{cases}
\end{equation}
In this way, interactive segmentation reuses exactly the same dense-fusion and sparse-query decoding interface as the other modes: it can stand on its own as a click-driven segmenter, or attach to Mode-1 and Mode-2 as a lightweight refinement step that injects user feedback without altering the backbone or decoder.

\subsubsection{Propagation-based volumetric segmentation}
Propagation is treated as an intra-volume special case of in-context segmentation. Rather than drawing support examples from an external annotated set, the reference set is built from previously segmented slices within the same 3D volume. Given a volume $\mathbf{V}=\{\mathbf{I}_z\}_{z=1}^{D}$ and a seed prediction $\hat{\mathbf{M}}_{z^*}$ initialized by any of the 2D modes (semantic, in-context, or interactive) or provided as ground truth, UniPro segments the remaining slices by repeatedly applying
\begin{equation}
    \hat{\mathbf{M}}_z = f_{\theta}(\mathbf{I}_z \mid \mathcal{R}_z),
\end{equation}
where $\mathcal{R}_z = \{(\mathbf{I}_{r}, \hat{\mathbf{M}}_{r})\}_{r \in \mathcal{N}(z)}$ denotes the intra-volume reference set formed by the seed slice together with nearby slices that have already been propagated.

Propagation therefore shares exactly the same reference-conditioned formulation as in-context segmentation, but with two distinguishing properties: the support examples are restricted to the current volume, and predictions are produced sequentially along the slice dimension. Notably, when the seed slice is itself initialized by the semantic class query $\mathbf{p}_n$, the entire pipeline becomes fully automatic, allowing UniPro to perform 3D volumetric segmentation without introducing any 3D-specific architecture. Further details on reference set construction and bidirectional inference are provided in the main text (Section~\ref{sec:propagation}).

\subsection{Model architectures}

\subsubsection{Shared slice-based image backbone.}
UniPro adopts a shared slice-based backbone to process both native 2D images and individual slices from 3D volumes within a unified framework. Given an input slice, the image encoder extracts a hierarchy of multi-scale visual features that serve as the common representation for all segmentation settings. Following standard medical segmentation practice, we use a lightweight UNet-style backbone~\citep{isensee2021nnu} and retain feature maps at three spatial resolutions, corresponding to $1/16$, $1/8$, and $1/4$ of the input size. These multi-scale features provide complementary semantic and spatial information, and are shared across semantic, in-context, interactive, and propagation-based segmentation.


\subsubsection{Reference mask and interaction encoding.}
To support reference-conditioned prediction, UniPro further encodes auxiliary conditioning inputs associated with the query slice. Specifically, reference masks are processed by a lightweight mask encoder that produces hierarchical features aligned with the image backbone at the same three resolutions ($1/4$, $1/8$, and $1/16$ of the input resolution). This design enables reference annotations to be fused with visual features in a spatially consistent manner. For interactive segmentation, user prompts are encoded with an analogous lightweight prompt encoder so that clicks or other interaction signals can be mapped into the same multi-scale feature space. By using aligned hierarchical encoders for reference masks and interaction prompts, UniPro preserves a unified conditioning interface across different segmentation settings.


\subsubsection{Decoder}
As shown in Figure~\ref{fig:decoder}, we use $L=6$ layers in total in the transformer decoder. Similar to Mask2Former~\citep{cheng2022masked}, we adopt a round-robin strategy for multi-scale interaction between image features and integrated queries. The blocks are scheduled in a cyclic order across scales (i.e., $1 \rightarrow 2 \rightarrow 3 \rightarrow 1 \rightarrow 2 \rightarrow 3$), ensuring balanced cross-scale information exchange throughout the decoding process. Finally, the decoded features are passed through a lightweight mask decoder that projects them back to the spatial resolution of the input, producing the final segmentation masks.

\begin{wrapfigure}{r}{0.48\textwidth}
\vspace{-10pt}
\centering
\includegraphics[width=0.95\linewidth]{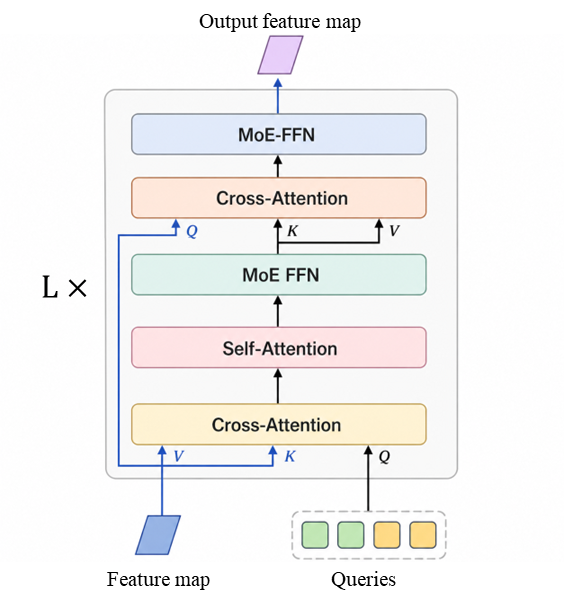}
\vspace{-10pt}
\caption{
Transformer decoder with MoE-FFN.
Queries attend to image features via cross-attention, followed by self-attention and MoE-FFN.
}
\label{fig:decoder}
\vspace{-10pt}
\end{wrapfigure}

\subsubsection{Mixture-of-Experts FFN Layers}

UniPro uses a shared transformer decoder across all segmentation modes. While the attention layers provide a common mechanism for exchanging information between query tokens and fused image features, different conditions, including semantic queries, reference-based object queries, and click queries, may require different nonlinear transformations. To improve mode adaptability without introducing separate decoders, we replace the standard feed-forward network (FFN) in the decoder with a Mixture-of-Experts FFN (MoE-FFN).

Each expert is implemented as an independent two-layer FFN:
\begin{equation}
    \mathrm{FFN}_m(\mathbf{x})
    =
    \mathbf{W}^{(2)}_m
    \sigma
    \left(
    \mathbf{W}^{(1)}_m \mathbf{x}
    \right),
    \quad m=1,\ldots,M,
\end{equation}
where $M$ is the number of experts, $\mathbf{W}^{(1)}_m$ and $\mathbf{W}^{(2)}_m$ are learnable linear projections of the $m$-th expert, and $\sigma(\cdot)$ denotes the activation function.

Given a query representation $\mathbf{Q}\in\mathbb{R}^{q\times C}$, a lightweight gating network predicts token-wise expert weights:
\begin{equation}
    \boldsymbol{\alpha}
    =
    \mathrm{softmax}
    \left(
    G(\mathbf{Q})
    \right),
    \quad
    \boldsymbol{\alpha}\in\mathbb{R}^{q\times M}.
\end{equation}
Each query token is processed by all experts, and their outputs are combined using the gating weights:
\begin{equation}
    \mathrm{MoE\text{-}FFN}(\mathbf{Q})_i
    =
    \sum_{m=1}^{M}
    \alpha_{i,m}
    \mathrm{FFN}_m(\mathbf{Q}_i),
    \quad
    i=1,\ldots,q.
\end{equation}
The output is then added to the input through a residual connection followed by layer normalization:
\begin{equation}
    \mathbf{Q}^{\prime}
    =
    \mathrm{LayerNorm}
    \left(
    \mathbf{Q}
    +
    \mathrm{MoE\text{-}FFN}(\mathbf{Q})
    \right).
\end{equation}

Only the FFN layers are expert-specialized, while the cross-attention and self-attention layers remain shared across all modes. This design keeps the decoder compact and unified, while allowing different query tokens to adaptively select suitable nonlinear transformations for heterogeneous conditioning signals.



\subsection{Implementation}
\label{appendix_training}
\subsubsection{Training details.}
We train UniPro with a two-stage training strategy. In the first stage, the model is trained for 100 epochs on single-slice segmentation tasks to learn robust mode-conditioned segmentation capability. Each training batch is randomly assigned to one of the three segmentation modes, with probabilities of 0.3, 0.3, and 0.4 for semantic segmentation, in-context segmentation, and interactive segmentation, respectively. This stage enables the shared backbone and decoder to learn a unified representation across different conditioning formats. In this stage, models are trained on 8 Quadro RTX 8000 GPUs using the AdamW optimizer with a base learning rate of $1\times10^{-4}$. We use a per-GPU batch size of 8, resulting in a total batch size of 64. A cosine annealing learning-rate scheduler with 10 warm-up epochs is adopted, where the minimum learning rate is scaled by $1\times10^{-5}$. During training, images are resized to $512\times512$, with data augmentations including random flipping, affine transformations, brightness and contrast adjustment, Gaussian blur, Gaussian noise, and grid distortion. The segmentation objective combines binary cross-entropy loss and Dice loss:
\begin{equation}
\mathcal{L}_{\mathrm{stage 1}} = \mathcal{L}_{\mathrm{2d}} = \mathcal{L}_{\mathrm{dice}} + \mathcal{L}_{\mathrm{ce}}.
\end{equation}

In the second stage, we fine-tune the entire model, including the image encoder, for 10 epochs with propagation-based segmentation. This adapts the shared representation to sequential prediction with multiple, potentially imperfect image--mask references. Since propagation training is performed on 3D datasets, directly fine-tuning on propagation batches may cause the model to forget the previously learned 2D segmentation capability and degrade its performance on native 2D datasets. To mitigate this issue, we interleave propagation training with regular single-slice training: each propagation batch is paired with a standard 2D training batch sampled with the same mode probabilities as in Stage 1. The propagation batch is constructed from an ordered slice queue, where the model predicts target slices conditioned on previous image-mask references. This stage enables the model to learn slice-wise mask propagation for volumetric segmentation while preserving its compatibility with native 2D inputs and the three original segmentation modes. 

In this stage, models are trained using the AdamW optimizer with a base learning rate of $5\times10^{-4}$. We use a per-GPU propagation batch size of 1. Propagation is performed on 3D chunks with a queue size of 8 slices. For each target slice, the reference window contains 5 image-mask pairs. When more than 5 previous predictions are available, we use a sliding reference window while always retaining the seed slice as an anchor. This anchor-pinned design provides a stable reference throughout propagation and helps reduce semantic drift across long slice sequences. For propagation training, data augmentation is applied consistently to the entire 3D chunk rather than independently to each slice. Specifically, all slices within the same queue undergo the same spatial transformations, ensuring that inter-slice anatomical correspondence is preserved during propagation learning.

The Stage-2 objective combines auxiliary 2D supervision and propagation-based volumetric losses:
\begin{equation}
\mathcal{L}_{\mathrm{stage2}} = \mathcal{L}_{2d} + \mathcal{L}_{\mathrm{prop}}
= \mathcal{L}_{2d} + 
\mathcal{L}_{\mathrm{fwd}}
+
\lambda_{\mathrm{bwd}}\mathcal{L}_{\mathrm{bwd}}
+
\lambda_{\mathrm{3d}}\mathcal{L}_{\mathrm{3d}} .
\end{equation}
Here, $\mathcal{L}_{2d}$ denotes the auxiliary single-slice supervision used to preserve 2D segmentation capability, while $\mathcal{L}_{\mathrm{fwd}}$ and $\mathcal{L}_{\mathrm{bwd}}$ correspond to forward and backward propagation losses. $\mathcal{L}_{\mathrm{3d}}$ is a volume-level consistency loss computed on stacked slice predictions. In our experiments, we set $\lambda_{\mathrm{bwd}}=0.7$ and $\lambda_{\mathrm{3d}}=5.0$.

\subsubsection{Inference details.}
During inference, each input image is resized so that its longer side equals 512 pixels while preserving the aspect ratio, and the predicted masks are remapped back to the original resolution before evaluation. For single-slice (2D) segmentation, UniPro directly predicts a mask for each input image using the chosen mode (semantic, in-context, or interactive), with the same per-mode forward pass as during training. For volumetric (3D) segmentation, we first select the middle slice of the target organ as the seed slice and initialize its mask via semantic, in-context, or interactive segmentation, depending on the evaluation setting. Starting from this seed, UniPro performs bidirectional propagation along the axial direction---forward to subsequent slices and backward to preceding slices---where each target slice is segmented using the anchor-pinned reference set containing the seed and nearby propagated slices. For the main evaluation, we compute Dice over slices that contain the target object. For 3D datasets, this corresponds to volumetric Dice computed over voxels within target-containing slices; for 2D datasets, Dice is computed on target-containing images/slices. This protocol is designed for a unified multi-mode evaluation setting. In interactive segmentation, user input is naturally defined on target-containing slices, since a positive click cannot be placed on an absent object and target-absent slices are typically skipped in practical annotation workflows. Applying the same target-present restriction across 2D and 3D datasets, and across semantic, in-context, and interactive modes, ensures that all modes are evaluated on the same foreground-relevant slice set. To further assess whether this foreground-region evaluation affects the conclusions, we additionally report full-volume Dice over all slices in Appendix~\ref{app:full-volume-vs-nonempty}. 


\subsection{Click encoding and simulation.}
User interactions are represented as binary disk maps with a radius of 1 pixel during both training and inference. To simulate corrective interactions, we follow the standard protocol of generating clicks from the discrepancy between the current prediction and the ground truth~\citep{sofiiuk2022reviving,liu2023simpleclick}. We place each new click at the centroid of the largest connected error region, so that the simulated clicks better reflect how users typically correct substantial local mistakes in medical images. We consider three click-based operating settings. In semantic or in-context initialization, the model first produces an initial segmentation and then performs two rounds of click-guided refinement. In pure interactive segmentation , the model applies three clicks sequentially, producing an updated mask after each interaction and thus yielding three intermediate refinement results.

For the single-click propagation comparison in
Table~\ref{tab1}, all methods receive the same initial
positive click on the same seed slice for each
volume--target pair. Each method generates its own
seed mask, whose quality is part of the end-to-end
click-to-volume evaluation. In multi-click evaluations,
the click-generation policy and budget are matched
within each setting, while corrective click locations
depend on each method's current prediction.

\subsection{Additional implementation details for baseline models}

All trainable baselines use the same dataset splits, annotations, and
preprocessing as UniPro, with matched training budgets and method-specific
input formats. Universal 2D baselines, including Hermes and Iris, are
jointly trained on the same multi-dataset training collection as UniPro.
In contrast, nnU-Net is trained separately for each dataset following its
standard self-configuring pipeline. Trainable propagation baselines use
the corresponding 3D training datasets. No additional medical training
data or annotations are introduced for UniPro.

Initialization is method-specific. For SAM2, we directly use the official
sam2.1\_hiera\_base\_plus checkpoint pretrained on large-scale natural
images, without any additional fine-tuning on medical data. SAM2 is
therefore excluded from the matched medical-training protocol above.
For MedSAM2, the encoder is initialized with MAE-pretrained ViT-Base
weights, after which the entire network is fine-tuned end-to-end on our
curated medical datasets. Other medical segmentation baselines, such as Hermes,
are trained from scratch on the same curated datasets. Released medical checkpoints are not used,
to avoid potential overlap with the evaluation data. Positive and negative
clicks are simulated consistently across methods using the policy in
Appendix~A.4.

\section{Theoretical Analysis of Bidirectional Supervision}
\label{app:theory}

We provide a simple idealized analysis to illustrate the intuition behind bidirectional supervision. This analysis is conditional on local stability assumptions and should be viewed as an explanatory perspective rather than a formal generalization guarantee.

\subsection{Setup and Notation}
Let $\mathbf{M}_0 \in \mathcal{Y}_0$ denote the ground-truth mask on the seed slice, represented as a probability map in $[0,1]^{H \times W}$; we use the Euclidean norm $\|\cdot\|$ throughout. Let $\mathcal{T}_t: \mathcal{Y}_t \to \mathcal{Y}_{t+1}$ be the true forward transfer between slice $t$ and $t{+}1$, so that $\mathbf{M}_{t+1} = \mathcal{T}_t(\mathbf{M}_t)$.

Our model is memory-based: a single reference-conditioned function $f(\mathbf{I}_q, \mathcal{R}_q)$ predicts the mask at a query slice $\mathbf{I}_q$ given a reference set $\mathcal{R}_q$ of previously segmented slices. Forward propagation at step $t$ corresponds to $P_t(\cdot) := f(\mathbf{I}_{t+1}, \mathcal{R}_t^{\mathrm{fwd}})$ with $\mathcal{R}_t^{\mathrm{fwd}}$ containing past forward predictions, while backward reconstruction at step $t$ is $Q_t(\cdot) := f(\mathbf{I}_t, \mathcal{R}_t^{\mathrm{bwd}})$ with $\mathcal{R}_t^{\mathrm{bwd}}$ containing future backward predictions. The operators $P_t$ and $Q_t$ use the same network $f$ but differ in query slice and reference-set content.

The predicted forward trajectory is $\hat{\mathbf{M}}_0 := \mathbf{M}_0$, $\hat{\mathbf{M}}_{t+1} := P_t(\hat{\mathbf{M}}_t)$ for $t = 0, \dots, n{-}1$. The backward reconstruction chain and its output are
\begin{equation}
Q^{(n:0)} := Q_0 \circ Q_1 \circ \cdots \circ Q_{n-1}, \qquad \tilde{\mathbf{M}}_0 := Q^{(n:0)}(\hat{\mathbf{M}}_n),
\end{equation}
and the bidirectional consistency loss is $\mathcal{L}_{\mathrm{bi}} := \|\tilde{\mathbf{M}}_0 - \mathbf{M}_0\|$.

\subsection{Assumptions}
\textbf{(A1) Bounded one-step forward error.} There exist $\delta_t \ge 0$ such that $\|P_t(\mathbf{M}_t) - \mathbf{M}_{t+1}\| \le \delta_t$ for all $t$.

\textbf{(A2) Lipschitz forward operator.} There exist $L_t \ge 0$ such that $\|P_t(\mathbf{a}) - P_t(\mathbf{b})\| \le L_t \|\mathbf{a} - \mathbf{b}\|$ for all $\mathbf{a}, \mathbf{b}$.

\textbf{(A3) Local non-collapse of the backward chain.} For intuition, assume that in a local neighborhood of the true endpoint $\mathbf{M}_n$, the backward chain is approximately non-collapsing and can reconstruct the seed mask up to error $\beta$. That is, there exist $\alpha > 0$ and $\beta \ge 0$ such that
\begin{align}
\|Q^{(n:0)}(\mathbf{z}) - Q^{(n:0)}(\mathbf{M}_n)\| &\ge \alpha \|\mathbf{z} - \mathbf{M}_n\|, \quad \forall \mathbf{z}, \tag{A3a} \\
\|Q^{(n:0)}(\mathbf{M}_n) - \mathbf{M}_0\| &\le \beta. \tag{A3b}
\end{align}
This assumption is not guaranteed globally for learned networks. It only formalizes the desirable regime in which backward reconstruction preserves endpoint discrepancies and penalizes forward trajectories that cannot be mapped back to the reliable seed slice.

\subsection{Main Result}

\begin{proposition}[Idealized endpoint-error bound]
\label{prop:bi-bound}
Under Assumption (A3),
\begin{equation}
\|\hat{\mathbf{M}}_n - \mathbf{M}_n\| \;\le\; \frac{\mathcal{L}_{\mathrm{bi}} + \beta}{\alpha}.
\label{eq:app-bi-bound}
\end{equation}
\end{proposition}

\begin{proof}
Applying (A3a) with $\mathbf{z} = \hat{\mathbf{M}}_n$,
\begin{equation}
\alpha \|\hat{\mathbf{M}}_n - \mathbf{M}_n\| \;\le\; \|Q^{(n:0)}(\hat{\mathbf{M}}_n) - Q^{(n:0)}(\mathbf{M}_n)\|.
\end{equation}
By the triangle inequality,
\begin{align}
\|Q^{(n:0)}(\hat{\mathbf{M}}_n) - Q^{(n:0)}(\mathbf{M}_n)\|
&\le \|Q^{(n:0)}(\hat{\mathbf{M}}_n) - \mathbf{M}_0\| + \|Q^{(n:0)}(\mathbf{M}_n) - \mathbf{M}_0\| \\
&\le \mathcal{L}_{\mathrm{bi}} + \beta.
\end{align}
Dividing by $\alpha$ yields Eq.~\eqref{eq:app-bi-bound}.
\end{proof}

\begin{proposition}[One-step accumulation bound]
\label{prop:one-step}
Under (A1)--(A2) with uniform constants $L_t \le L$ and $\delta_t \le \delta$,
\begin{equation}
\|\hat{\mathbf{M}}_n - \mathbf{M}_n\| \;\le\; \delta \sum_{i=0}^{n-1} L^i.
\end{equation}
\end{proposition}

\begin{proof}
Let $e_t := \|\hat{\mathbf{M}}_t - \mathbf{M}_t\|$. Then
\begin{equation}
e_{t+1} \le \|P_t(\hat{\mathbf{M}}_t) - P_t(\mathbf{M}_t)\| + \|P_t(\mathbf{M}_t) - \mathbf{M}_{t+1}\| \le L_t e_t + \delta_t.
\end{equation}
With $e_0 = 0$, unrolling gives $e_n \le \delta \sum_{i=0}^{n-1} L^i$.
\end{proof}

\begin{corollary}[Combined conditional bound]
Under (A1)--(A3),
\begin{equation}
\|\hat{\mathbf{M}}_n - \mathbf{M}_n\| \;\le\; \min\!\left\{ \delta \sum_{i=0}^{n-1} L^i, \; \frac{\mathcal{L}_{\mathrm{bi}} + \beta}{\alpha} \right\}.
\end{equation}
\end{corollary}

\paragraph{Interpretation.}
This analysis should be understood as an idealized intuition rather than a strict guarantee. The first term, $\delta \sum_{i=0}^{n-1} L^i$, describes standard forward error accumulation, where local errors may grow along long propagation chains. The second term, $(\mathcal{L}_{\mathrm{bi}}+\beta)/\alpha$, suggests that, when the backward chain is locally stable and non-collapsing, a smaller bidirectional reconstruction loss is consistent with a smaller endpoint error. The key takeaway is that bidirectional supervision acts as a cycle-consistency regularizer: it penalizes forward trajectories whose accumulated errors prevent reconstruction back to the reliable seed slice. Thus, while it does not universally guarantee low propagation error, it provides an additional long-range constraint that helps explain the empirical reduction of drift in volumetric propagation.

\section{Dataset}
\subsection{Dataset Details}
\label{sec:appdataset}

We use a diverse collection of public medical image segmentation datasets covering 3D CT/MRI volumes and native 2D images across abdominal, cardiac, brain, skin, breast, lung, and endoscopic targets. 
Below we briefly describe each dataset used in training and evaluation.

\paragraph{AbdomenCT-1K~\citep{Ma-2021-AbdomenCT-1K}.} A large-scale abdominal CT benchmark with expert annotations for the liver, kidney, spleen, and pancreas, aggregated from multiple medical centers. To prevent severe class imbalance with other datasets in the training pool, we use only a subset of the released data, split into 352 training, 20 validation, and 76 test volumes.

\paragraph{Combined (CT-MR) Healthy Abdominal Organ Segmentation (CHAOS)~\citep{CHAOSdata2019, CHAOS2021}.} A multi-modal abdominal dataset from the ISBI 2019 CHAOS Challenge, comprising CT and MRI scans of healthy subjects with expert annotations for the liver, right kidney, left kidney, and spleen. We use only the MRI subset, split into 36 training, 9 validation, and 15 test volumes.

\paragraph{Multi-Centre, Multi-Vendor \& Multi-Disease Cardiac Image Segmentation Challenge (M\&Ms)~\citep{campello2021multi}.} A multi-centre, multi-vendor cardiac MRI dataset from the MICCAI 2020 Challenge, comprising short-axis cine scans of patients with cardiomyopathies and healthy controls. Expert annotations of the left ventricle, right ventricle, and left ventricular myocardium are provided at end-diastolic and end-systolic phases. We follow the official train/test split.

\paragraph{The Liver Tumor Segmentation Benchmark (LiTS)~\citep{bilic2023liver_lits}}. A multi-center benchmark of 201 contrast-enhanced abdominal CT scans (131 training, 70 testing) covering patients with diverse liver tumors, including hepatocellular carcinoma and metastases from colorectal, breast, and lung cancers, with expert annotations for liver and tumor regions. We use the 131 public cases with an 80\%/20\% split, retaining only the tumor labels.

\paragraph{The Kidney Tumor Segmentation 2019 dataset (KiTS) ~\citep{heller2019kits19}.} A dataset of CT scans from 300 patients who underwent nephrectomy for kidney tumors at the University of Minnesota Medical Center (2010--2018), of which 210 cases are publicly released. We split the public portion 80\%/20\% for training and testing, retaining only the tumor labels.

\paragraph{Dataset of Breast Ultrasound Images (Breast)~\citep{al2020dataset}.} A collection of 780 grayscale breast ultrasound images from 600 female patients (aged 25--75) at Baheya Hospital, Egypt, categorized as normal, benign, or malignant, with expert-annotated lesion masks. We use only the 647 benign and malignant cases with an 80\%/20\% train/test split.

\paragraph{Chest X-ray Masks and Labels (Chest)~\citep{jaeger2013automatic}.} A chest radiograph dataset jointly released by the National Library of Medicine (Bethesda, USA) and Shenzhen No.3 People's Hospital (China), covering a broad range of abnormalities such as effusions and miliary patterns. We use only the lung segmentation masks, with an 80\%/20\% train/test split.

\paragraph{BKAI-IGH NeoPolyp Dataset (BKAI)~\citep{ngoc2021neounet}.} A colonoscopy dataset of 1,200 images (1,000 white-light and 200 FICE) released by the BKAI Research Center and the Institute of Gastroenterology and Hepatology, Vietnam, with expert-verified segmentation masks distinguishing neoplastic from non-neoplastic polyps. We adopt an 80\%/20\% split on the official training set of 1,000 images.

\paragraph{International Skin Imaging Collaboration (ISIC)~\citep{tschandl2018ham10000, codella2019skin}.} A large-scale dermoscopic image collection from the ISIC Skin Lesion Analysis Challenges, with expert-annotated lesion boundary masks spanning diverse skin tones, lighting conditions, and lesion morphologies. We use the ISIC 2018 segmentation subset and follow the official train/validation/test split.

\paragraph{Beyond the Cranial Vault (BTCV)~\citep{landman2015miccai}.} An abdominal CT dataset with multi-organ annotations, used as an external 3D test set to evaluate cross-benchmark generalization. To align the label space with AbdomenCT-1K, we merge the left and right kidneys into a single class and retain only the organs present in AbdomenCT-1K (liver, kidney, spleen, and pancreas).

\paragraph{Automatic Cardiac Diagnosis Challenge (ACDC)~\citep{bernard2018acdc}.} A cardiac MRI dataset acquired at the University Hospital of Dijon with 1.5T and 3.0T Siemens scanners, providing short-axis cine sequences annotated for the same structures as M\&Ms at end-systolic and end-diastolic phases. For validation, we use 200 ED/ES phase volumes from the 100 patients in the official training set; for final evaluation, we use 100 phase volumes from the 50 patients in the official test set.

\paragraph{University of Waterloo Skin Cancer (UW-SC)~\citep{uwaterloo_skincancer}.} A collection of 167 dermoscopic skin lesion images with expert-verified segmentation masks, used solely as an external test set.

\paragraph{Breast Ultrasound Dataset B (BUS)~\citep{yap2017automated}.} A breast ultrasound dataset of 163 images with expert-annotated lesion regions. We merge benign and malignant lesions into a single class and use it solely as an external test set.

\paragraph{Brain Tumor Segmentation (BraTS)~\citep{baid2021rsna}.} A multi-modal brain tumor MRI dataset, used as an unseen 3D test set to evaluate generalization to an anatomical target and imaging domain absent from the training pool. We use only the T2-weighted modality. The original annotations contain three tumor sub-regions (enhancing tumor, tumor core, and peritumoral edema); to reduce the difficulty of this unseen-class evaluation, we merge them into a single tumor region label.

\subsection{Dataset statistics}

Table~\ref{tab:dataset} summarizes the datasets used to train and evaluate our framework. We draw from a diverse collection of publicly available medical image segmentation benchmarks covering a wide range of imaging modalities and anatomical regions. The training pool comprises nine datasets spanning abdominal CT/MRI (AbdomenCT-1K, CHAOS, LiTS, KiTS), cardiac MRI (M\&Ms), ultrasound (Breast Cancer), dermoscopy (ISIC), endoscopy (BKAI), and chest X-ray. These datasets differ substantially in scale—ranging from a few hundred to several thousand 3D volumes or 2D slices—as well as in annotation granularity, from single-organ masks to multi-organ delineations. To rigorously assess generalization, we further include held-out test sets covering both external domains (BTCV, ACDC, UW-SC, BUS), which share class definitions with the training data but originate from different cohorts or acquisition centers, and unseen classes (BraTS), which evaluate the framework's ability to segment anatomical structures absent from training. This split allows us to disentangle cross-center robustness from cross-class transferability.

\begin{table}[h!]
    \centering
    \scriptsize
    \caption{Datasets statistics.}
    \label{tab:dataset}
    \begin{tabular}{lcccccccc}
        \toprule
        \multirow{2}{*}{\textbf{Dataset}} & \multirow{2}{*}{\textbf{\# cls}} & \multirow{2}{*}{\textbf{Modality}} & \multicolumn{3}{c}{\textbf{3D Volumes}} & \multicolumn{3}{c}{\textbf{2-D Slices}}  \\
        \cmidrule(lr){4-6} \cmidrule(lr){7-9} 
        & & & \textbf{Train} & \textbf{Validate} & \textbf{Test} & \textbf{Train} & \textbf{Validate} & \textbf{Test} \\
        \midrule
        AbdomenCT-1K\citep{Ma-2021-AbdomenCT-1K}& 4 & CT & 352 & 20 & 76 & 39956 & 2557 & 8665 \\
        CHAOS\citep{CHAOSdata2019, CHAOS2021} & 4 & MRI & 36 & 9 & 15 & 4984 & 1309 & 1945 \\
        M\&Ms~\citep{campello2021multi} & 3 & MRI & 300 & - & 340 & 2475 & - & 2821 \\
        LiTS~\citep{bilic2023liver_lits} & 1 & CT & 104 & - & 27 & 5687 & - & 1625 \\
        KiTS~\citep{heller2019kits19} & 1 & CT & 168 & - & 42 & 5269 & - & 1609 \\
        Chest X-ray~\citep{al2020dataset, jaeger2013automatic} & 1 & X\_ray & - & - & - & 450 & - & 116\\
        Breast Cancer~\citep{al2020dataset} & 1 & Ultrasound & - & - & - & 518 & - & 129\\
        BKAI~\citep{ngoc2021neounet} & 2  & Endoscopy & - & - & - & 800 & - & 200  \\
        ISIC~\citep{tschandl2018ham10000, codella2019skin} & 1  & Dermoscopy & - & - & - & 2594 & 100 & 1000 \\
        \midrule
        BTCV~\citep{landman2015miccai} & 4 & CT & - & - & 30 & - & - & 3791 \\
        ACDC~\citep{bernard2018acdc} & 3  & MRI & - & 200 & 100 & - & 1841
        & 1001 \\
        UW-SC~\citep{uwaterloo_skincancer} & 1  & Dermoscopy & - & - & - & - & - & 167\\
        BUS~\citep{al2020dataset} & 1  & Ultrasound & - & - & - & - & - & 163\\
        \midrule
        BraTS~\citep{baid2021rsna} & 1 & MRI  & -- & -- & 74 & - & -& 4761\\
        \bottomrule
    \end{tabular}
\end{table}

\subsection{Data preprocessing}
\label{app:data_preprocessing}
Preprocessing is dataset-specific and retains both isotropic
and anisotropic volumes. AbdomenCT-1K, CHAOS, BTCV, and KiTS
are resampled to an isotropic spacing of
$1.5\times1.5\times1.5$\,mm. M\&Ms and ACDC are resampled
in-plane to $1.5\times1.5$\,mm while preserving their native
through-plane spacing of 5--10\,mm, yielding strongly
anisotropic volumes. LiTS uses
$0.7676\times0.7676\times1.0$\,mm spacing and is mildly
anisotropic, whereas BraTS uses isotropic
$1.0\times1.0\times1.0$\,mm spacing. These spacings refer
to the volumes before the final 2D input resizing.
For intensity preprocessing, CT volumes are clipped to
$[-991,500]$\,HU for AbdomenCT-1K and BTCV, and
$[-200,250]$\,HU for KiTS and LiTS. MRI volumes from CHAOS,
M\&Ms, ACDC, and BraTS are clipped to their volume-wise
0th--99th intensity percentiles. All volumes are then
rescaled to $[0,255]$ using volume-wise min--max normalization.
UniPro processes ordered slices without explicit voxel-spacing
or physical-distance inputs; isotropic spacing is therefore
not required, although large through-plane gaps may weaken
cross-slice correspondence.

\section{More Experiments}

\subsection{Effect of flexible interactive correction}
\label{appendix:interactive-vs-propagate}

We compare two interactive 3D segmentation paradigms under the same annotation budget: \emph{slice-wise interaction}, where the user corrects foreground slices independently, and \emph{click-and-propagate}, where clicks on selected slices are propagated to the rest of the volume. Both protocols use the same robot click generator, the same per-slice click cap, the same total click budget, and the same volume-level stopping criterion. Therefore, differences in annotation cost reflect the interaction protocol rather than the click oracle.

\paragraph{Robot click generation.}
Clicks are generated by a deterministic robot user. For each slice, the first click is sampled from the largest connected component of the foreground ground truth. Subsequent clicks are placed at the centroid of the largest connected component in the current 2D error map, computed between the thresholded prediction and the ground truth. A click is labeled positive if the selected error component is a false-negative region, and negative if it is a false-positive region. No human input is used at any stage.

\paragraph{Common budget and stopping rules.}
For each $(\text{volume}, \text{organ})$ pair, both protocols are given a total budget of $B=100$ clicks. Within a single slice, at most $C=5$ clicks can be used. A slice-level click loop stops early once the 2D Dice reaches target $T$. For the seed slice in the propagation protocol, we use a stricter threshold $\tau_{\text{2D}}=0.95$ to obtain a cleaner reference mask. The full $(\text{volume}, \text{organ})$ evaluation stops when either the volume-level 3D Dice reaches the dataset-specific target $T$, or the click budget $B$ is exhausted. We set $T=0.90$ for ACDC and $T=0.80$ for BraTS20.

\paragraph{Slice-wise interaction.}
This protocol represents conventional per-slice correction. The user first provides one positive click on every foreground slice, producing an initial stack of independent 2D predictions. The method then enters a refinement loop. At each step, we select the slice with the largest 2D error among slices that have not reached the per-slice click cap, and add clicks according to the robot click rule. The process repeats until the 3D Dice target $T$ is reached or the total budget is exhausted. Importantly, no information is shared across slices: each slice is segmented and refined independently.

\paragraph{Click-and-propagate.}
This protocol evaluates the proposed propagation-based annotation workflow. The seed slice is chosen as the median foreground slice. The robot user places up to $C=5$ clicks on the seed slice, with early stopping at $\tau_{\text{2D}}=0.95$. The resulting seed mask is then used as the initial reference to propagate predictions to the remaining slices. If the propagated volume does not reach the target 3D Dice, the protocol enters an iterative refinement loop. At each iteration, we select the unclicked slice with the largest 2D error, add up to $C=5$ clicks on that slice, and pin every clicked slice so that its corrected mask is not overwritten by later propagation. Propagation is then re-run using the top-$K=5$ pinned slices, ranked by their 2D Dice, as references. The loop continues until the same termination conditions are met: reaching the dataset-specific 3D Dice target $T$ or exhausting the $B=100$ click budget.

\paragraph{Datasets and evaluated pairs.}
We evaluate on the full test split of ACDC and BraTS. For ACDC, we evaluate $100$ volumes with three cardiac organs, resulting in $300$ $(\text{volume}, \text{organ})$ pairs, with target $T=0.90$. For BraTS20, we evaluate $74$ volumes with one tumor class, resulting in $74$ pairs, with target $T=0.80$. We use a lower target for BraTS because it is an unseen dataset. For each dataset, we report two annotation-cost metrics. The first is the \textbf{Number of Clicks (NoC)}, defined as the number of clicks required to first reach the target 3D Dice $T$, censored at $B=100$ if the target is not reached. This censoring makes the mean NoC conservative. The second is the \textbf{number of edited slices}, defined as the number of distinct slices that receive at least one click. NoC measures direct click effort, while the number of edited slices captures slice-level scrolling and inspection effort, which is an important practical cost in volumetric annotation. Both metrics are averaged across all evaluated pairs in each dataset.

\subsection{Full-Volume Dice versus Target-Containing-Slice Dice}
\label{app:full-volume-vs-nonempty}

Our main experiments report Dice on target-containing slices to maintain a unified evaluation protocol across 2D/3D data and semantic, in-context, and interactive modes. We additionally compute standard full-volume Dice over all slices. As shown in Table~\ref{tab:full_volume_vs_nonempty_avg}, the two metrics remain close across datasets, indicating that target-containing-slice evaluation does not substantially inflate performance in most cases. Larger gaps appear on KiTS and LiTS, where targets are smaller, sparser, or more prone to axial false positives; tumor cases are similarly more sensitive because small leakage into target-absent slices can noticeably reduce full-volume Dice. These results support using target-containing-slice Dice as a foreground-region metric for interactive evaluation, while reporting full-volume Dice as the standard volumetric metric.

\begin{table}[t]
\centering
\small
\setlength{\tabcolsep}{8pt}
\renewcommand{\arraystretch}{1.08}
\caption{
Comparison between standard full-volume Dice (\%) and target-containing-slice Dice.
Results are averaged over organs within each dataset when applicable.
$\Delta$ denotes target-containing-slice Dice minus full-volume Dice, reported in percentage points.
}
\label{tab:full_volume_vs_nonempty_avg}
\begin{tabular}{l c c c}
\toprule
Dataset & Full-vol Dice & Target-slice Dice & $\Delta$ \\
\midrule
Abd-1K & 86.11 & 88.62 & +2.51 \\
CHAOS  & 87.05 & 89.70 & +2.65 \\
M\&M   & 86.06 & 87.64 & +1.58 \\
LiTS   & 63.32 & 66.95 & +3.63 \\
KiTS   & 73.05 & 76.95 & +3.90 \\
BTCV   & 82.19 & 84.40 & +2.21 \\
ACDC   & 87.80 & 87.97 & +0.17 \\
\bottomrule
\end{tabular}
\end{table}

\subsection{Robustness to Seed-Slice Selection}
\label{app:seed_robustness}

Our default propagation protocol selects the seed slice as the median axial index among slices containing the target organ. This choice provides a simple and deterministic initialization. To examine whether the propagation performance depends critically on this seed location, we conduct a controlled robustness study on ACDC by selecting seed slices at different relative positions within the target-organ slice range. In this analysis, the selected seed slice is initialized with the ground-truth mask for all seed choices. This removes the confounding effect of seed-slice segmentation quality and isolates the effect of seed location on propagation. Specifically, we evaluate seeds at the $20\%$, $40\%$, $50\%$, $60\%$, and $80\%$ axial quantiles of the target-organ extent, as well as a deterministic random seed slice sampled independently for each volume. All remaining slices are then segmented by the same propagation procedure.

As shown in Table~\ref{tab:seed_position_gt}, UniPro is not highly sensitive to the exact seed-slice location when the seed mask is fixed to the ground truth. Across different seed choices, the mean Dice remains in a narrow range from $89.46$ to $91.03$. Seeds near the two ends of the target-organ extent, such as $q=0.2$ and $q=0.8$, show slightly larger drops, which is expected because boundary seeds provide less balanced context for covering the full organ range. In contrast, seeds around the middle of the organ extent remain stable, with $q=0.4$ even slightly outperforming the median seed. These results support using the median slice as a simple, deterministic, and standardized choice for evaluation, rather than an oracle condition on which the method critically depends.

\begin{table}[t]
\centering
\small
\setlength{\tabcolsep}{7pt}
\renewcommand{\arraystretch}{1.08}
\caption{
Seed-position robustness on ACDC under ground-truth seed-mask initialization.
GT-init removes seed-slice segmentation error and isolates the effect of seed location on propagation.
All results are 3D Dice on target-containing slices.
}
\label{tab:seed_position_gt}
\begin{tabular}{l c c c c c c}
\toprule
Seed slice & $q=0.2$ & $q=0.4$ & $q=0.5$ & $q=0.6$ & $q=0.8$ & Random \\
\midrule
Mean Dice & 89.81 & 91.03 & 90.98 & 90.53 & 89.46 & 90.12 \\
$\Delta$ vs. median & -1.17 & +0.05 & 0.00 & -0.45 & -1.52 & -0.86 \\
\bottomrule
\end{tabular}
\end{table}

\subsection{Statistical Analysis of Propagation Performance}
\label{app:propagation_statistics}

We assess the statistical reliability of the propagation
results by reporting variability across evaluated
volume--target pairs and paired comparisons with MedSAM2.
Table~\ref{tab:propagation_statistics}(a) summarises UniPro's
mean Dice, standard deviation, and 95\% confidence intervals
computed using Student's $t$-distribution. Paired $t$-tests
show significant improvements on five of the eight datasets
at the unadjusted $p<0.05$ level
(Table~\ref{tab:propagation_statistics}(b)), while differences
on CHAOS, LiTS, and KiTS are not statistically significant.
The wider confidence intervals on LiTS, KiTS, and BraTS
indicate greater uncertainty in these settings.
All statistics are based on a single training run per
method and describe test-pair variability rather than
variation across training runs.

\begin{table}[t]
\centering
\caption{Statistical analysis of propagation performance.
(a) UniPro's mean Dice (\%), SD, and 95\% CI.
(b) Paired comparisons with MedSAM2, where $\Delta$ denotes
the Dice difference in percentage points.
Bold $p$-values indicate significance at the unadjusted
$p<0.05$ level.}
\label{tab:propagation_statistics}
\small
\renewcommand{\arraystretch}{1.12}
\begin{minipage}[t]{0.55\linewidth}
\centering
\textbf{(a) UniPro performance}\par\smallskip
\setlength{\tabcolsep}{4pt}
\begin{tabular}{@{}lrrc@{}}
\toprule
Dataset & Mean & SD & 95\% CI \\
\midrule
Abd-1K & 88.62 & 12.64 & [87.19, 90.05] \\
CHAOS  & 89.70 & 9.05  & [87.37, 92.04] \\
M\&Ms  & 87.64 & 10.31 & [87.01, 88.28] \\
LiTS   & 66.95 & 28.60 & [55.15, 78.76] \\
KiTS   & 76.95 & 18.94 & [71.04, 82.85] \\
BTCV   & 84.40 & 14.70 & [81.32, 87.48] \\
ACDC   & 87.97 & 13.06 & [86.49, 89.45] \\
BraTS  & 54.89 & 23.03 & [49.56, 60.23] \\
\bottomrule
\end{tabular}
\end{minipage}\hfill
\begin{minipage}[t]{0.43\linewidth}
\centering
\textbf{(b) UniPro versus MedSAM2}\par\smallskip
\setlength{\tabcolsep}{5pt}
\begin{tabular}{@{}lrc@{}}
\toprule
Dataset & $\Delta$ & $p$-value \\
\midrule
Abd-1K & $+2.33$ & $\mathbf{0.0165}$ \\
CHAOS  & $-0.94$ & $0.5254$ \\
M\&Ms  & $+4.67$ & $\mathbf{5.49\times10^{-11}}$ \\
LiTS   & $+4.70$ & $0.3136$ \\
KiTS   & $+3.63$ & $0.1911$ \\
BTCV   & $+3.85$ & $\mathbf{0.0035}$ \\
ACDC   & $+3.79$ & $\mathbf{0.0066}$ \\
BraTS  & $+14.15$ & $\mathbf{1.14\times10^{-4}}$ \\
\bottomrule
\end{tabular}
\end{minipage}
\end{table}

We additionally compare UniPro with Sli2Vol+ and Vol2Flow on
AbdomenCT-1K liver propagation using the same seed slice and
ground-truth seed mask. UniPro achieves 94.03\% Dice, compared
with 86.05\% for Sli2Vol+ and 83.81\% for Vol2Flow
(Table~\ref{tab:additional_propagation}).

\begin{table}[t]
\centering
\caption{Liver propagation on AbdomenCT-1K with the same seed slice
and ground-truth seed mask for all methods. Dice scores are percentages.}
\label{tab:additional_propagation}
\small
\setlength{\tabcolsep}{12pt}
\renewcommand{\arraystretch}{1.12}
\begin{tabular}{@{}lc@{}}
\toprule
Method & Dice ($\uparrow$) \\
\midrule
Vol2Flow & 83.81 \\
Sli2Vol+ & 86.05 \\
UniPro (GT seed) & \textbf{94.03} \\
\bottomrule
\end{tabular}
\end{table}

\subsection{Inference Runtime and Memory}
\label{app:runtime_memory}

We compare UniPro with a 3D nnU-Net of similar parameter count on
AbdomenCT-1K test volumes using a single NVIDIA Quadro RTX 8000 GPU
and a batch size of 1. The 3D nnU-Net uses $128^3$ sliding-window
patches at 1.5\,mm isotropic spacing, while UniPro processes chunks
of eight 3-channel $512\times512$ slices. Table~\ref{tab:runtime_memory}
reports median per-volume inference time and peak GPU memory under
these native input configurations. UniPro uses approximately 46\% less peak memory but takes longer per
volume (28.0\,s versus 9.3\,s), reflecting a memory--runtime trade-off.
Sequential propagation limits parallelization because later predictions
depend on earlier outputs. These measurements compare the two native
deployment configurations and do not imply a general reduction in
computation or inference latency.

\begin{table}[t]
\centering
\caption{Inference runtime and memory on AbdomenCT-1K using one
NVIDIA Quadro RTX 8000 GPU and batch size 1. Time is the median
per volume; memory is peak GPU memory during inference.}
\label{tab:runtime_memory}
\small
\setlength{\tabcolsep}{6pt}
\renewcommand{\arraystretch}{1.12}
\begin{tabular}{@{}lcccc@{}}
\toprule
Model & Parameters & Native input & Time (s) & Memory (GB) \\
\midrule
3D nnU-Net & 31.84M & $128^3$ patch & 9.3 & 3.60 \\
UniPro & 32.84M & $8\times3\times512\times512$ chunk & 28.0 & 1.93 \\
\bottomrule
\end{tabular}
\end{table}

\subsection{Contributions of Stage-I Pretraining and Stage-II Propagation Training}
\label{app:training_stage_ablation}

To assess the contributions of the two training stages, we train all variants using only the cardiac subset of our training data and evaluate on ACDC. We compare Stage I only, Stage II only without Stage-I pretraining, and the full two-stage pipeline. Apart from the ablated stage, all other training and evaluation settings are kept identical across variants. At inference time, every variant receives the same seed slice and ground-truth seed mask for each volume--target pair, isolating propagation performance from differences in initialization quality. This cardiac-only ablation uses a different protocol from Table~\ref{tab:ablation_combined}(b), which evaluates 2D interleaving under multi-dataset training with semantic seed predictions.

\begin{table}[t]
\centering
\small
\setlength{\tabcolsep}{10pt}
\renewcommand{\arraystretch}{1.12}
\caption{Training-stage ablation using only cardiac training data. All variants are evaluated on ACDC with the same ground-truth seed masks. Dice scores are percentages.}
\label{tab:training_stage_ablation}
\begin{tabular}{@{}lc@{}}
\toprule
Training strategy & ACDC Dice (GT seed) \\
\midrule
Stage I only & 87.29 \\
Stage II only, without Stage-I pretraining & 88.12 \\
Full two-stage UniPro & \textbf{90.98} \\
\bottomrule
\end{tabular}
\end{table}

As shown in Table~\ref{tab:training_stage_ablation}, Stage I alone enables propagation, while Stage II alone can learn propagation without unified pretraining. Combining the two stages achieves the highest Dice, improving over Stage I only and Stage II only by 3.69 and 2.86 Dice points, respectively. These results support the benefit of combining unified pretraining with propagation-specific adaptation in this controlled cardiac setting.

\section{Visualization}
\subsection{Comparison with propagation baselines}
We further illustrate the behavior of UniPro on pancreas segmentation, one of the most challenging targets in abdominal imaging due to its small size, elongated and curved shape, low contrast against surrounding soft tissues, and high inter-patient variability in both location and morphology. As shown in Figure~\ref{fig:pancreas_qualitative}, all methods are initialized from a single positive click on the seed slice ($z=34$) and required to propagate the resulting mask through the entire volume. UniPro recovers a coherent pancreas across distant slices ($z=22$ to $z=48$) and reaches a 3D Dice of 0.834, while propagation baselines fail in characteristic ways: SAM2 (0.410) leaks heavily into adjacent organs and tissues; nnInteractive (0.445), despite being a dedicated 3D interactive model, loses the target on slices far from the seed; iSegFormer (0.503) and PAM (0.541) exhibit progressive boundary inflation as propagation proceeds; and MedSAM2 (0.665), the strongest medical propagation baseline, still under-segments the head and tail of the pancreas. The 3D reconstructions in the rightmost column further highlight that UniPro produces a continuous, anatomically plausible pancreas, whereas baselines show fragmented or distorted volumes. This case illustrates that the reference-conditioned propagation in UniPro remains stable on small, low-contrast, elongated structures where slice-by-slice drift typically dominates.

\begin{figure*}[t]
\centering
\includegraphics[width=\textwidth]{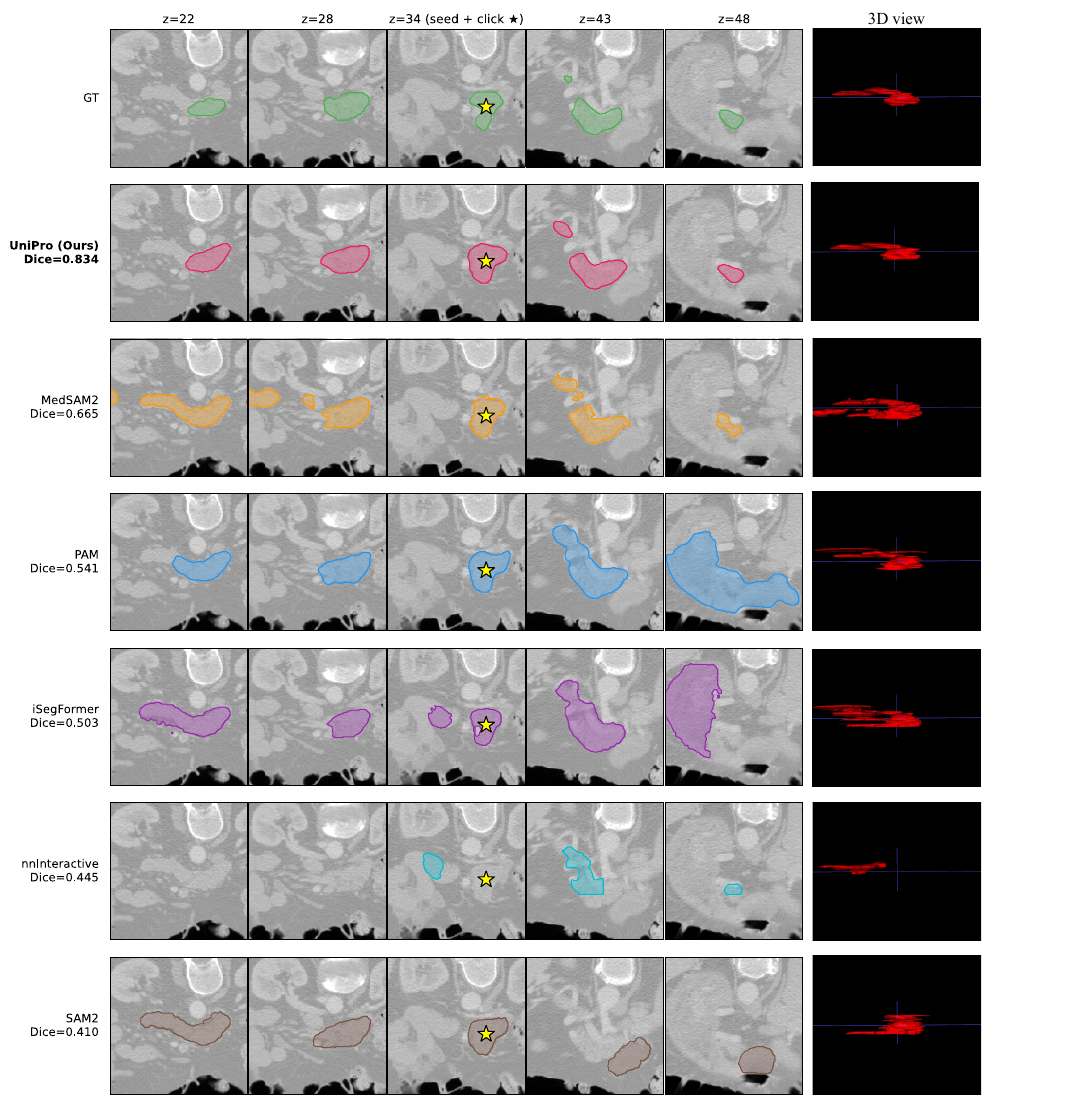}
\caption{
Qualitative comparison on pancreas segmentation under the single-click propagation setting. Each row shows a different method propagating from a single positive click ($\bigstar$) on the seed slice ($z=34$) to representative slices at $z=22, 28, 43, 48$, with the resulting 3D reconstruction shown in the rightmost column. UniPro produces masks that closely follow the ground-truth pancreas boundary across the entire volume, while baselines exhibit boundary leakage (PAM, iSegFormer), partial coverage (nnInteractive, MedSAM2), or large over-segmentation toward neighboring organs (SAM2). Per-volume 3D Dice is reported next to each method name.
}
\label{fig:pancreas_qualitative}
\end{figure*}

\subsection{Success and failure cases on tumor datasets}
\label{app:tumor_qualitative}

Figure~\ref{fig:qualitative_multicomponent} presents selected success and
failure cases on LiTS, KiTS, and BraTS. Each example shows the initialized
seed slice, predictions on four additional slices at increasing absolute
slice-index distances from the seed, and the corresponding 3D surfaces.
The successful LiTS and KiTS examples recover the dominant tumor component
across slices despite its irregular shape. The successful BraTS example
similarly captures the main lesion, although neither brain anatomy nor
brain tumors are included in training. These examples nevertheless retain
local errors, including missed small foci and false-positive extensions.

The failure cases highlight limitations on fragmented and heterogeneous
targets. In LiTS and KiTS, predictions cover part of the initialized region
but miss small, spatially separated components on more distant slices.
The BraTS failure captures only part of a heterogeneous lesion and omits
substantial surrounding and disconnected regions. Together, these
observations suggest that appearance consistency and cross-slice
traceability influence propagation beyond whether the anatomy was seen
during training. They motivate improved handling of newly appearing
components and weak correspondence between neighboring slices.

\begin{figure*}[p]
\centering
\includegraphics[width=\textwidth]{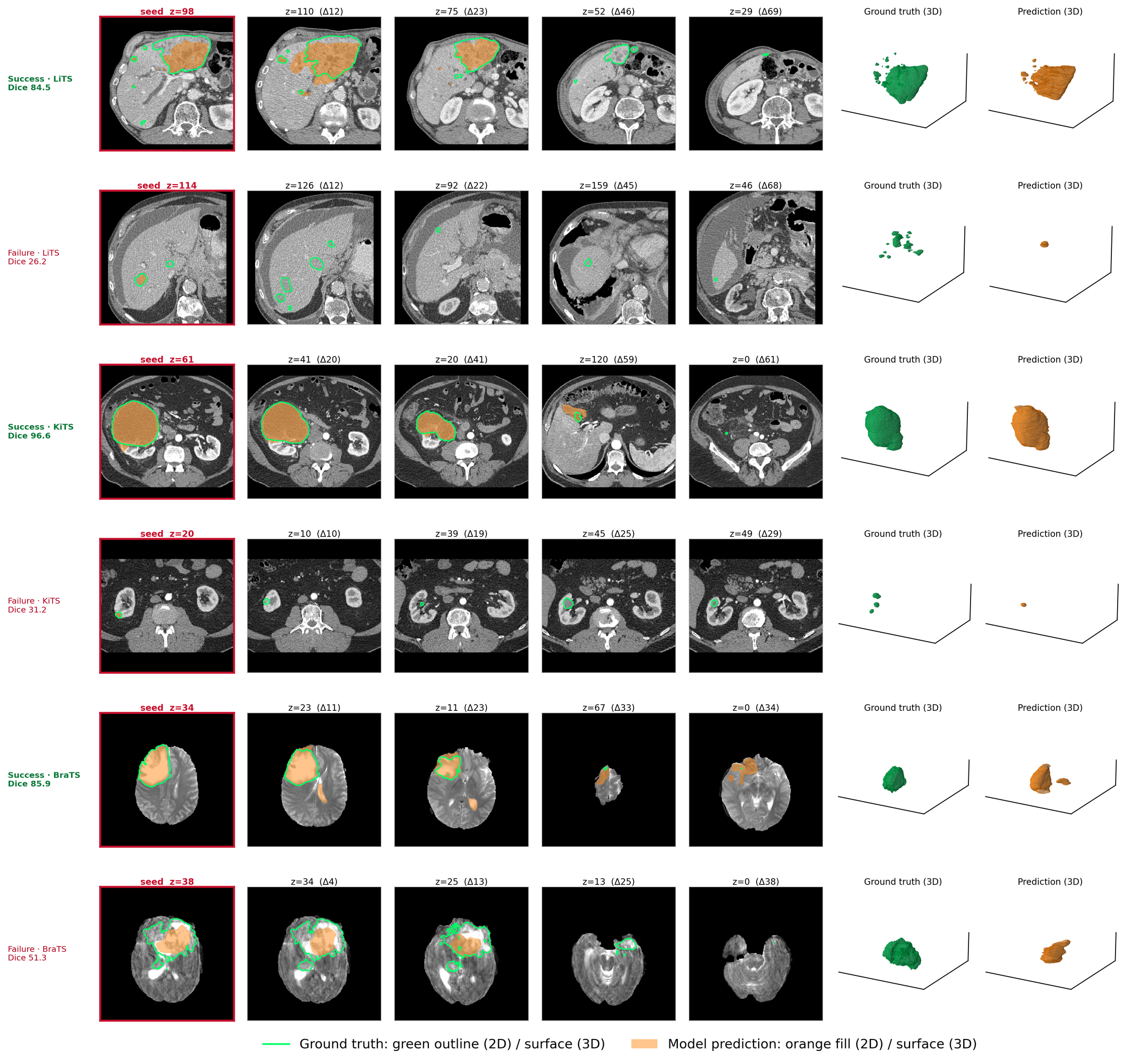}
\caption{Selected success and failure cases on LiTS, KiTS, and BraTS.
Each row shows the seed slice (red border), four additional slices ordered
by absolute distance from the seed, and 3D ground-truth and predicted
surfaces. Green contours and orange overlays denote ground truth and
predictions; 3D surfaces use the same colors. $z$ is the slice index and
$\Delta$ the absolute index distance in the preprocessed volume.}
\label{fig:qualitative_multicomponent}
\end{figure*}

\clearpage

\section{The Use of Large Language Models (LLMs)}
In this work, large language models (LLMs) were used as general-purpose assistants for code development, manuscript writing, and figure illustration. Specifically, LLMs aided in routine coding tasks (e.g., debugging, boilerplate generation, and refactoring), language polishing and formatting of the manuscript. They did not contribute to research ideation, experimental design, analysis, or result interpretation.


\clearpage

\section*{NeurIPS Paper Checklist}

\begin{enumerate}

\item {\bf Claims}
    \item[] Question: Do the main claims made in the abstract and introduction accurately reflect the paper's contributions and scope?
    \item[] Answer: \answerYes{} 

\item {\bf Limitations}
    \item[] Question: Does the paper discuss the limitations of the work performed by the authors?
    \item[] Answer: \answerYes{} 

\item {\bf Theory assumptions and proofs}
    \item[] Question: For each theoretical result, does the paper provide the full set of assumptions and a complete (and correct) proof?
    \item[] Answer: \answerYes{} 

    \item {\bf Experimental result reproducibility}
    \item[] Question: Does the paper fully disclose all the information needed to reproduce the main experimental results of the paper to the extent that it affects the main claims and/or conclusions of the paper (regardless of whether the code and data are provided or not)?
    \item[] Answer: \answerYes{} 

\item {\bf Open access to data and code}
    \item[] Question: Does the paper provide open access to the data and code, with sufficient instructions to faithfully reproduce the main experimental results, as described in supplemental material?
    \item[] Answer: \answerYes{} 

\item {\bf Experimental setting/details}
    \item[] Question: Does the paper specify all the training and test details (e.g., data splits, hyperparameters, how they were chosen, type of optimizer) necessary to understand the results?
    \item[] Answer: \answerYes{} 

\item {\bf Experiment statistical significance}
    \item[] Question: Does the paper report error bars suitably and correctly defined or other appropriate information about the statistical significance of the experiments?
    \item[] Answer: \answerYes{} 

\item {\bf Experiments compute resources}
    \item[] Question: For each experiment, does the paper provide sufficient information on the computer resources (type of compute workers, memory, time of execution) needed to reproduce the experiments?
    \item[] Answer: \answerYes{} 
    
\item {\bf Code of ethics}
    \item[] Question: Does the research conducted in the paper conform, in every respect, with the NeurIPS Code of Ethics \url{https://neurips.cc/public/EthicsGuidelines}?
    \item[] Answer: \answerYes{} 

\item {\bf Broader impacts}
    \item[] Question: Does the paper discuss both potential positive societal impacts and negative societal impacts of the work performed?
    \item[] Answer: \answerNA{} 

\item {\bf Safeguards}
    \item[] Question: Does the paper describe safeguards that have been put in place for responsible release of data or models that have a high risk for misuse (e.g., pre-trained language models, image generators, or scraped datasets)?
    \item[] Answer: \answerNA{} 

\item {\bf Licenses for existing assets}
    \item[] Question: Are the creators or original owners of assets (e.g., code, data, models), used in the paper, properly credited and are the license and terms of use explicitly mentioned and properly respected?
    \item[] Answer: \answerNA{} 

\item {\bf New assets}
    \item[] Question: Are new assets introduced in the paper well documented and is the documentation provided alongside the assets?
    \item[] Answer: \answerNA{} 

\item {\bf Crowdsourcing and research with human subjects}
    \item[] Question: For crowdsourcing experiments and research with human subjects, does the paper include the full text of instructions given to participants and screenshots, if applicable, as well as details about compensation (if any)? 
    \item[] Answer: \answerNA{} 

\item {\bf Institutional review board (IRB) approvals or equivalent for research with human subjects}
    \item[] Question: Does the paper describe potential risks incurred by study participants, whether such risks were disclosed to the subjects, and whether Institutional Review Board (IRB) approvals (or an equivalent approval/review based on the requirements of your country or institution) were obtained?
    \item[] Answer: \answerNA{} 

\item {\bf Declaration of LLM usage}
    \item[] Question: Does the paper describe the usage of LLMs if it is an important, original, or non-standard component of the core methods in this research? Note that if the LLM is used only for writing, editing, or formatting purposes and does \emph{not} impact the core methodology, scientific rigor, or originality of the research, declaration is not required.
    \item[] Answer: \answerYes{} 

\end{enumerate}

\end{document}